%% file: main.tex
\documentclass[a4paper,11pt]{article}
\usepackage[margin=1in]{geometry}
\usepackage{import}
\usepackage{subfig}

\include{defpack}
\makeatletter
\newcommand{\printfnsymbol}[1]{%
  \textsuperscript{\@fnsymbol{#1}}%
}
\makeatother
\begin{document}
\title{\bf {\color{black}On the Benefits of Multiple Gossip Steps in Communication-Constrained Decentralized Optimization} }
\date{}
\author[1]{Abolfazl Hashemi}
\author[1]{Anish Acharya \thanks{Equal Contribution}}
\author[1]{Rudrajit Das \printfnsymbol{1}}
\author[1]{\\Haris Vikalo}
\author[1]{Sujay Sanghavi}
\author[1,2]{Inderjit Dhillon}
\affil[1]{University of Texas at Austin}
\affil[2]{Amazon}
\maketitle
\begin{abstract}
\input{abstract}
\end{abstract}
\section{Introduction}
\input{intro}

\input{back}
\section{Convergence Analysis}
\input{analysis}
\begin{figure*}[!htb]
\centering 
\subfloat[$\gamma = 0.05$]{
    \label{fig:1_a}
	\includegraphics[width=0.31\textwidth]{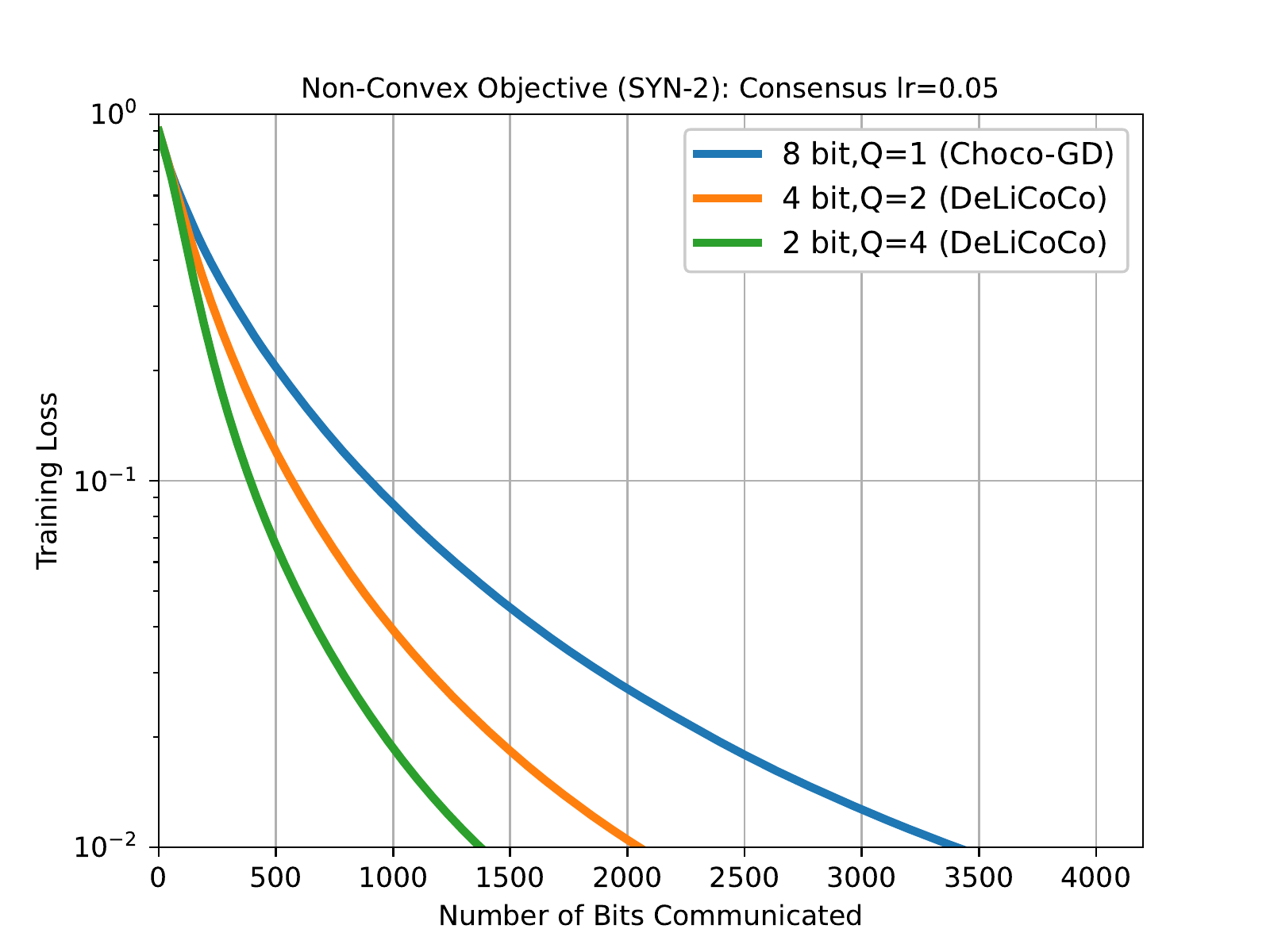}
	} 
\hspace{-0.3cm}
\subfloat[$\gamma = 0.1$]{
    \label{fig:1_b}
	\includegraphics[width=0.31\textwidth]{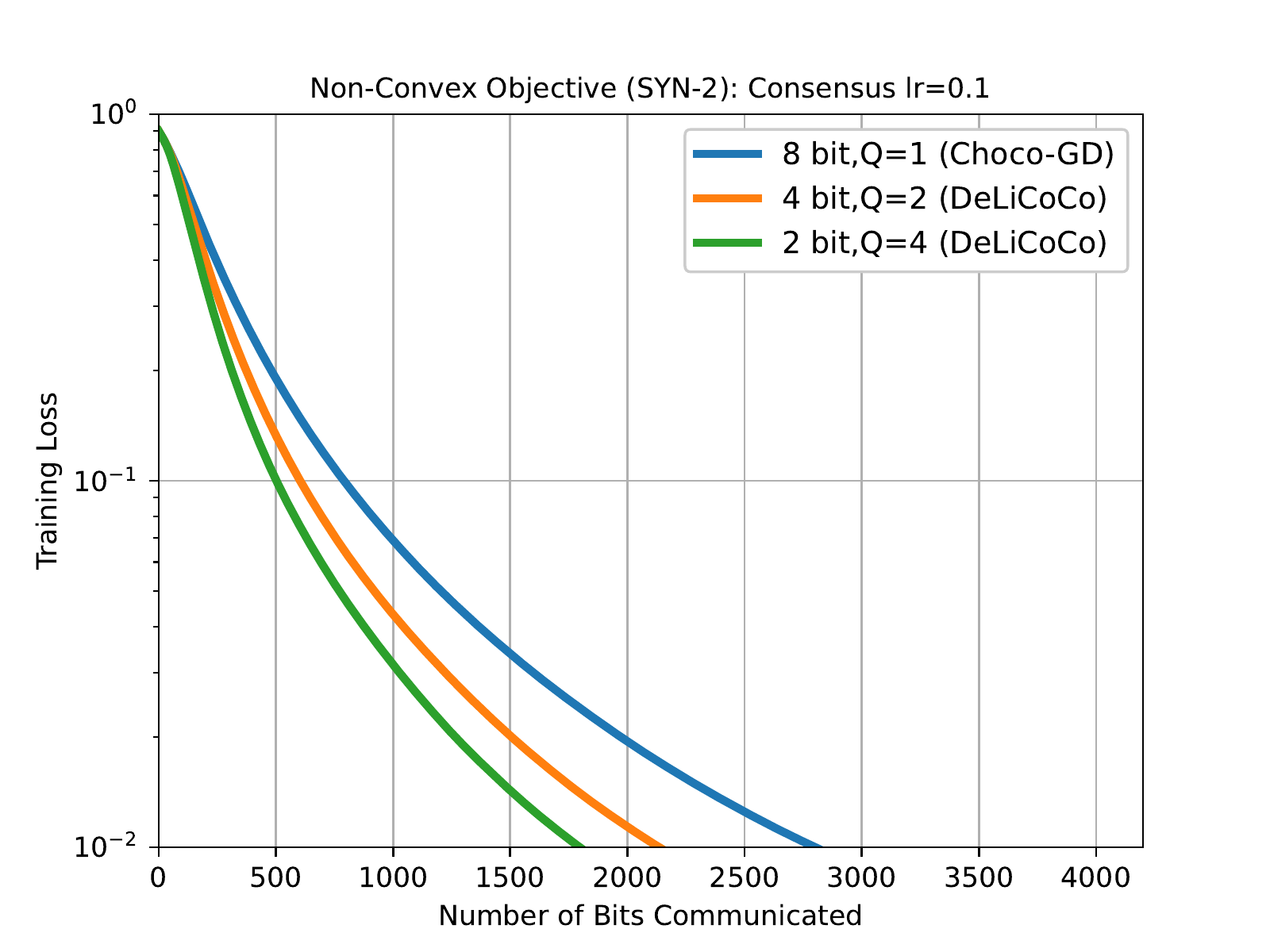}
	} 
\hspace{-0.3cm}
\subfloat[$\gamma = 0.15$]{
    \label{fig:1_c}
	\includegraphics[width=0.31\textwidth]{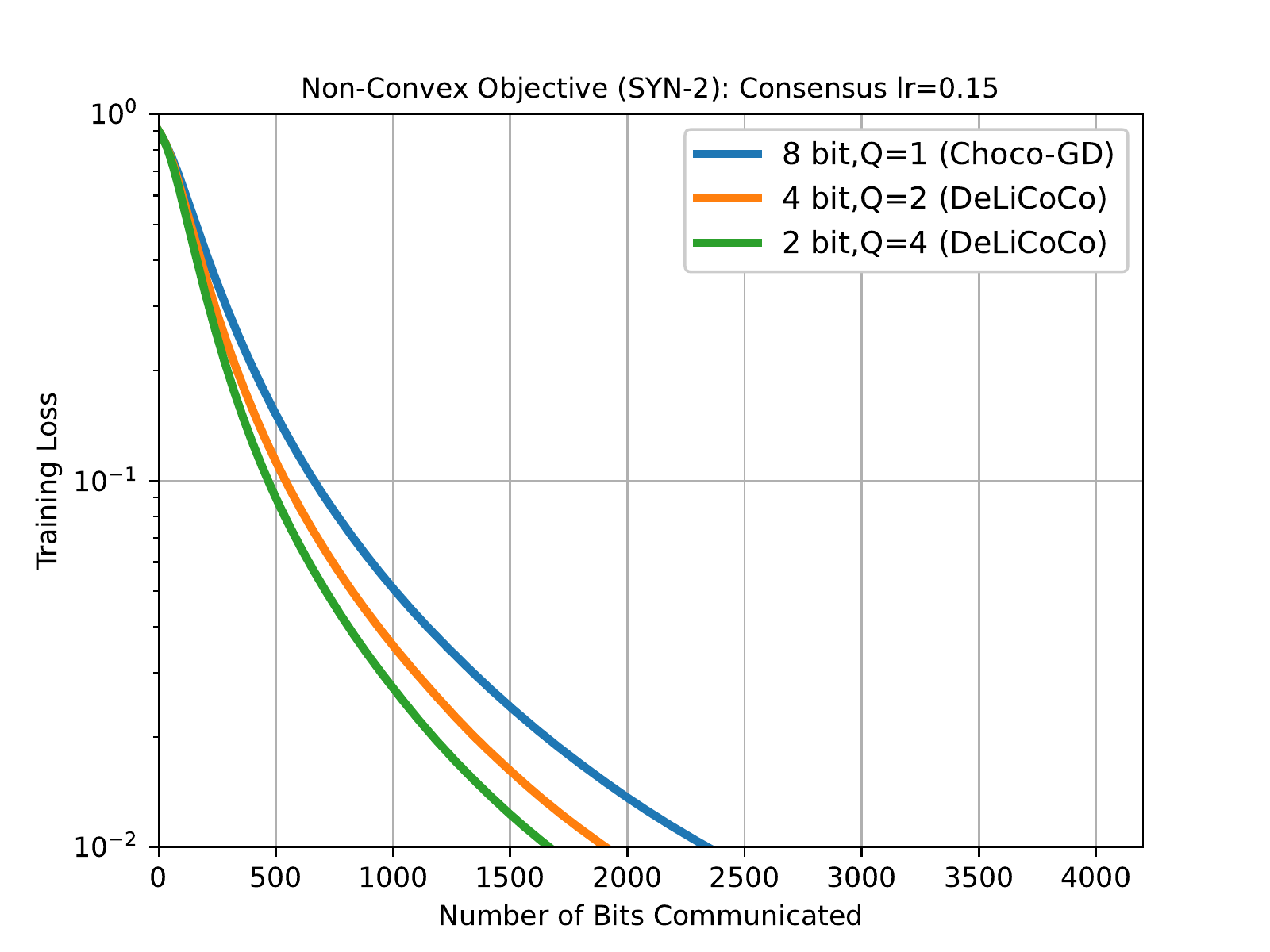}
	} 
\caption{\footnotesize Effect of different $(Q,b)$ pairs (where $b$ denotes the number of bits in qsgd) such that $Q b = 8$, on the total number of bits communicated for SYN-2, with three different consensus learning rates $\gamma$. 
In all three plots, torus topology is used with $n = 16$, $\ell_2$ regularization value = 0.001, and $\eta = 0.1$.}
\label{fig:1}
\end{figure*}
\begin{figure*}[t]
\centering 
\subfloat[$n = 9$]{
    \label{fig:2_a}
	\includegraphics[width=0.31\textwidth]{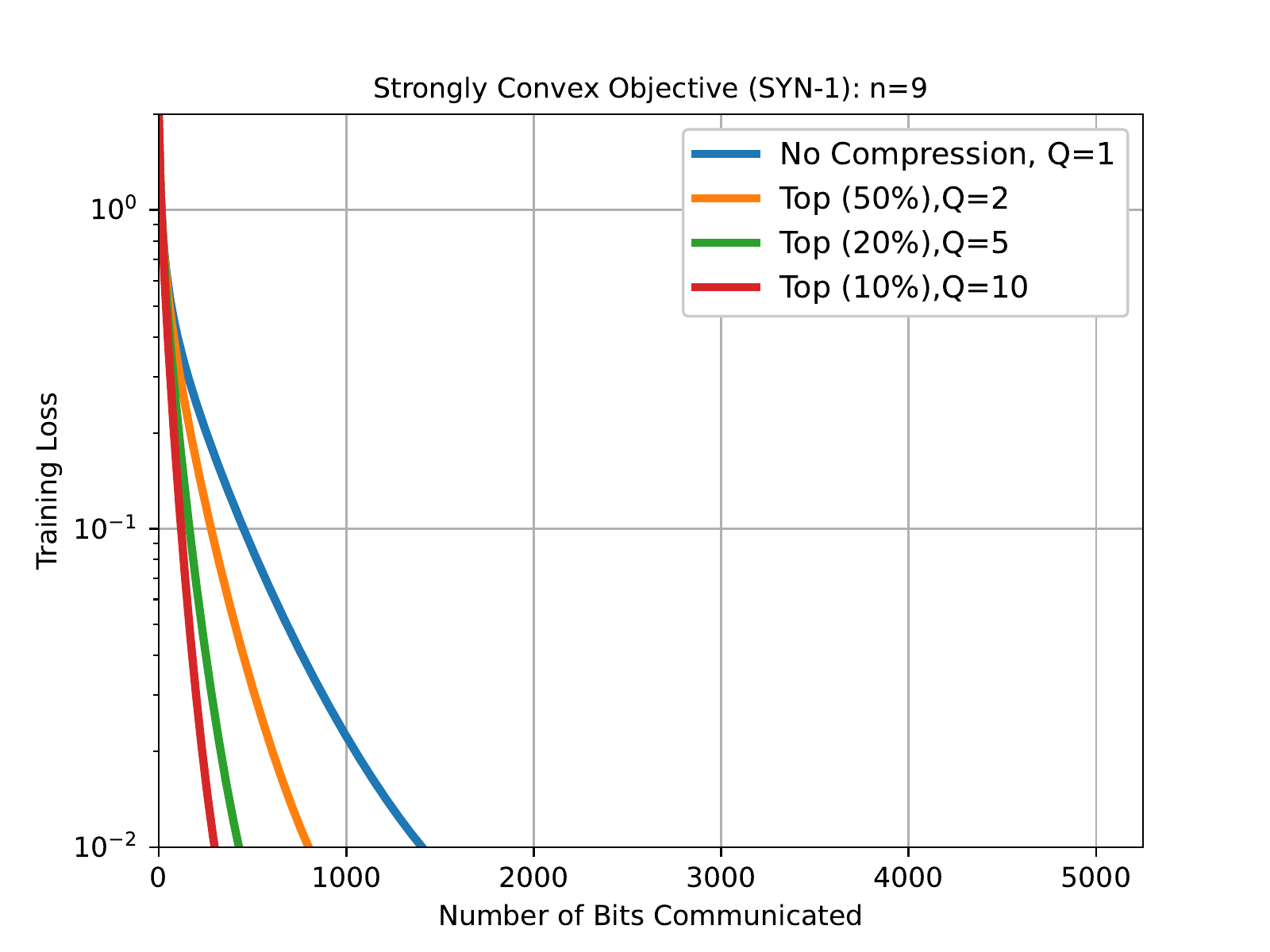}
	} 
\hspace{-0.3cm}
\subfloat[$n = 16$]{
    \label{fig:2_b}
	\includegraphics[width=0.31\textwidth]{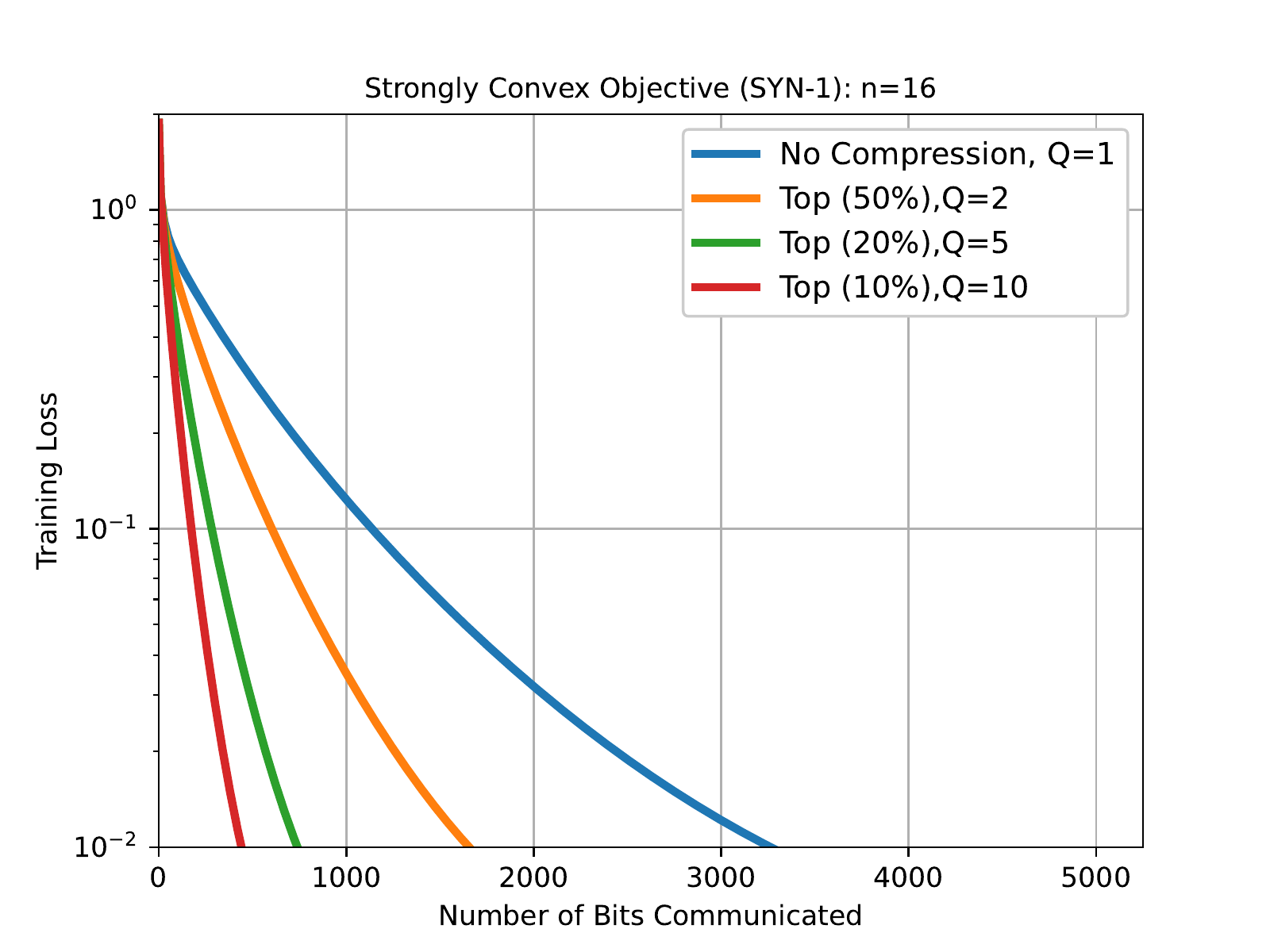}
	} 
\hspace{-0.3cm}
\subfloat[$n = 25$]{
    \label{fig:2_c}
	\includegraphics[width=0.31\textwidth]{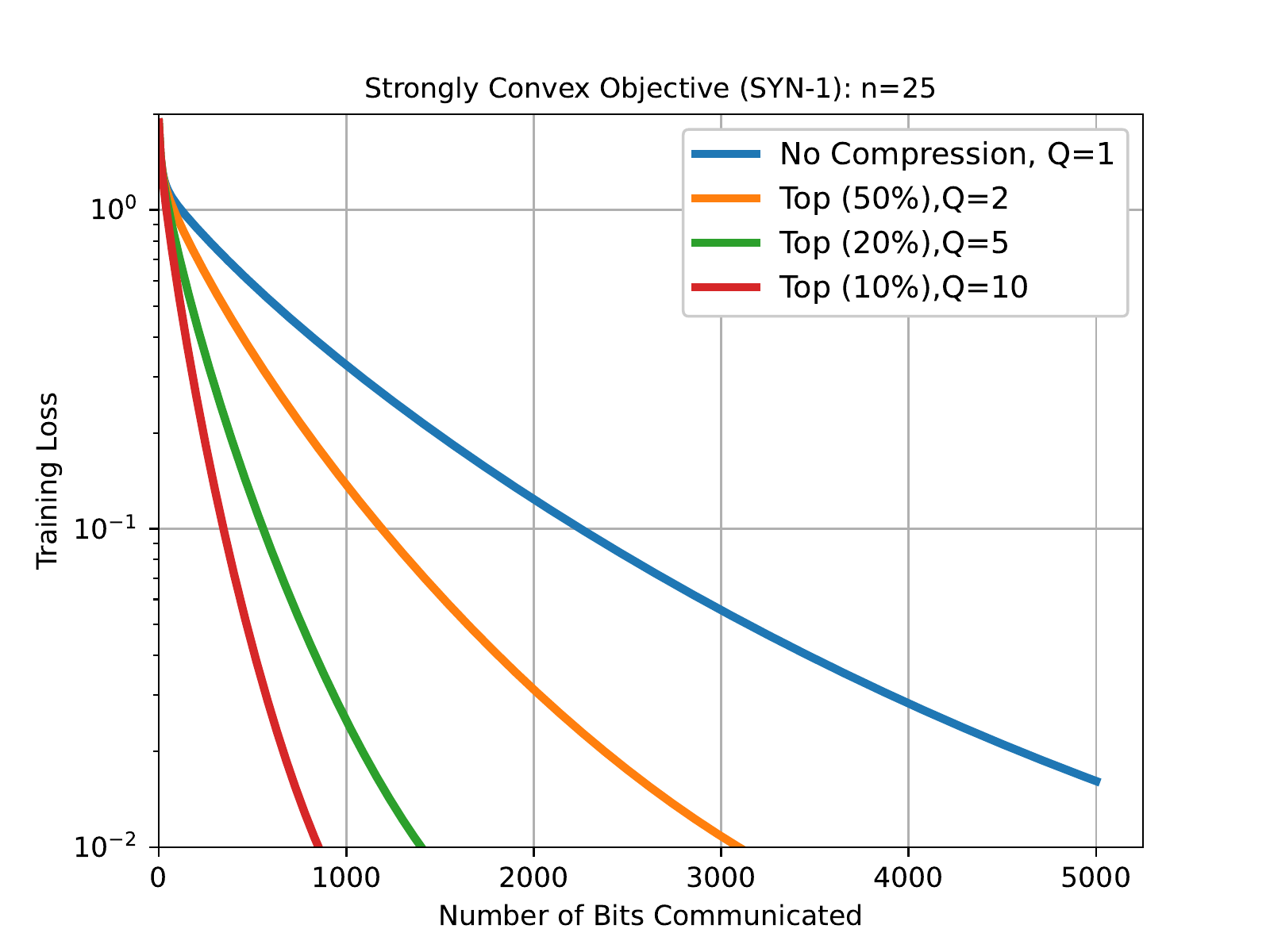}
	} 
\caption{\footnotesize Effect of different $(Q,\omega)$ pairs (where $\omega$ denotes the percentage of largest magnitude co-ordinates retained in the top-$k$ quantization) such that $Q \omega = 100$, on the total number of bits communicated for SYN-1. We consider the torus topology with three different values of $n$.
In all three plots, $\gamma = 0.05$, $\ell_2$ regularization value = 0.001, and $\eta = 0.1$.}
\label{fig:2}
\end{figure*}
\section{Verifying the Theory}\label{sec:exp}
\input{new_experiments}

\section{Conclusion}
We proposed DeLi-CoCo, a simple and communication efficient collaborative learning algorithm that achieves linear convergence in smooth nonconvex learning tasks that satisfy the PL condition, under arbitrary communication compression.
As part of the future work, it would be of interest to consider extensions to directed and time-varying networks and use of stochastic local gradients.
Another extension would be incorporating momentum to have an accelerated version of the proposed method.
\clearpage
\section*{A. Proofs}
\input{proof}
\clearpage
\bibliography{bib}
\bibliographystyle{acm}
\end{document}

%% file: defpack.tex
\usepackage[utf8]{inputenc} 
\usepackage[T1]{fontenc}    
\usepackage{hyperref}       
\usepackage{nicefrac}       
\usepackage{microtype}      
\usepackage{cite}
\usepackage{xcolor}
\usepackage{framed}
\colorlet{shadecolor}{pink}
\usepackage{authblk}

\usepackage{graphicx}
\usepackage{soul}
\usepackage{booktabs} 
\usepackage{tablefootnote}

\usepackage{amsmath,amsthm,amssymb,amsfonts}
\usepackage{algorithm}
\usepackage{algorithmic}
\usepackage{enumerate}
\usepackage{cleveref}
\usepackage{comment}
\usepackage{bm}
\usepackage{pifont}
\newcommand{\cmark}{\ding{51}}%
\newcommand{\xmark}{\ding{55}}%

\theoremstyle{plain}
\newtheorem{theorem}{Theorem}[]
\newtheorem{corollary}{Corollary}[theorem]
\newtheorem{lemma}[]{Lemma}

\newtheorem{assumption}{Assumption}[]

\newtheorem{remark}{Remark}[]

\newcommand{\ra}[1]{\renewcommand{\arraystretch}{#1}}

\def\x{{\mathbf x}}
\def\z{{\mathbf z}}

\def\s{{\mathbf s}}

\def\y{{\mathbf y}}

\def\E{{\mathbb E}}
\def\O{{\mathcal O}}
\def\R{{\mathbb{R}}}

\newcommand{\ts}{\textsuperscript}

%% file: abstract.tex
{\color{black}In decentralized optimization, it is common algorithmic practice 
to have nodes interleave (local) gradient descent iterations with gossip (i.e. averaging over the network) steps. Motivated by the training of large-scale machine learning models, it is also increasingly common to require that messages be {\em lossy compressed} versions of the local parameters. In this paper we show that, in such compressed decentralized optimization settings, there are benefits to having {\em multiple} gossip steps between subsequent gradient iterations, even when the cost of doing so is appropriately accounted for e.g. by means of reducing the precision of compressed information. In particular, we show that having $O(\log\frac{1}{\epsilon})$ gradient iterations {with constant step size} - and $O(\log\frac{1}{\epsilon})$ gossip steps between every pair of these iterations - enables convergence to within $\epsilon$ of the optimal value for smooth non-convex objectives satisfying Polyak-\L{}ojasiewicz condition. This result also holds for smooth strongly convex objectives. 
To our knowledge, this is the first work that derives convergence results for nonconvex optimization under arbitrary communication compression.\footnote{Throughout the paper we have used gossip/consensus, graph/network and node/client/agent  interchangeably.}
}

%% file: intro.tex
We consider distributed optimization  over a network with $n$ client nodes where  the objective function is possibly nonconvex. Formally, we are interested in 
\begin{equation}\label{eq:problem}
\min_{\mathbf{x}\in \R^d} \left[f(\mathbf{x}):=\sum_{i=1}^n f_i(\mathbf{x})\right],
\end{equation}
where $f_i: \mathbb{R}^d \to \mathbb{R}$ for $i \in [n]:=\left\{1, ..., n \right\}$ is the local objective function of the $i\ts{th}$ client. The goal of the clients in the network is to collaboratively solve the above optimization problem by passing messages over a graph that connects them \cite{nedic2009distributed}.

The optimization task in \eqref{eq:problem} arises in many important distributed machine learning (ML) tasks, i.e., training and optimization of ML models in a distributed/decentralized manner \cite{mcmahan2017communication,seaman2017optimal,assran2018stochastic}. Solving such distributed tasks is often facilitated by communication of agents' local model parameters over a network that governs their communication capabilities. Compared to a centralized optimization framework, distributed optimization enables locality of data storage and model updates which in turn offers computational advantages by delegating computations to multiple clients, and further promotes preservation of privacy of user information \cite{mcmahan2017communication}.

As the size of ML models grows, exchanging information across the network becomes a major challenge in distributed optimization \cite{assran2018stochastic}. It is therefore imperative to design communication-efficient strategies which reduce the amount of communicated data by performing compressed communication while at the same time, despite the use of compressed communication, achieve a convergence properties that is on par with the performance of centralized and distributed methods utilizing uncompressed information \cite{he2018cola,assran2018stochastic,koloskova2019decentralized}. 
\begin{figure}[t]
  \begin{center}
    \includegraphics[width=0.5\textwidth]{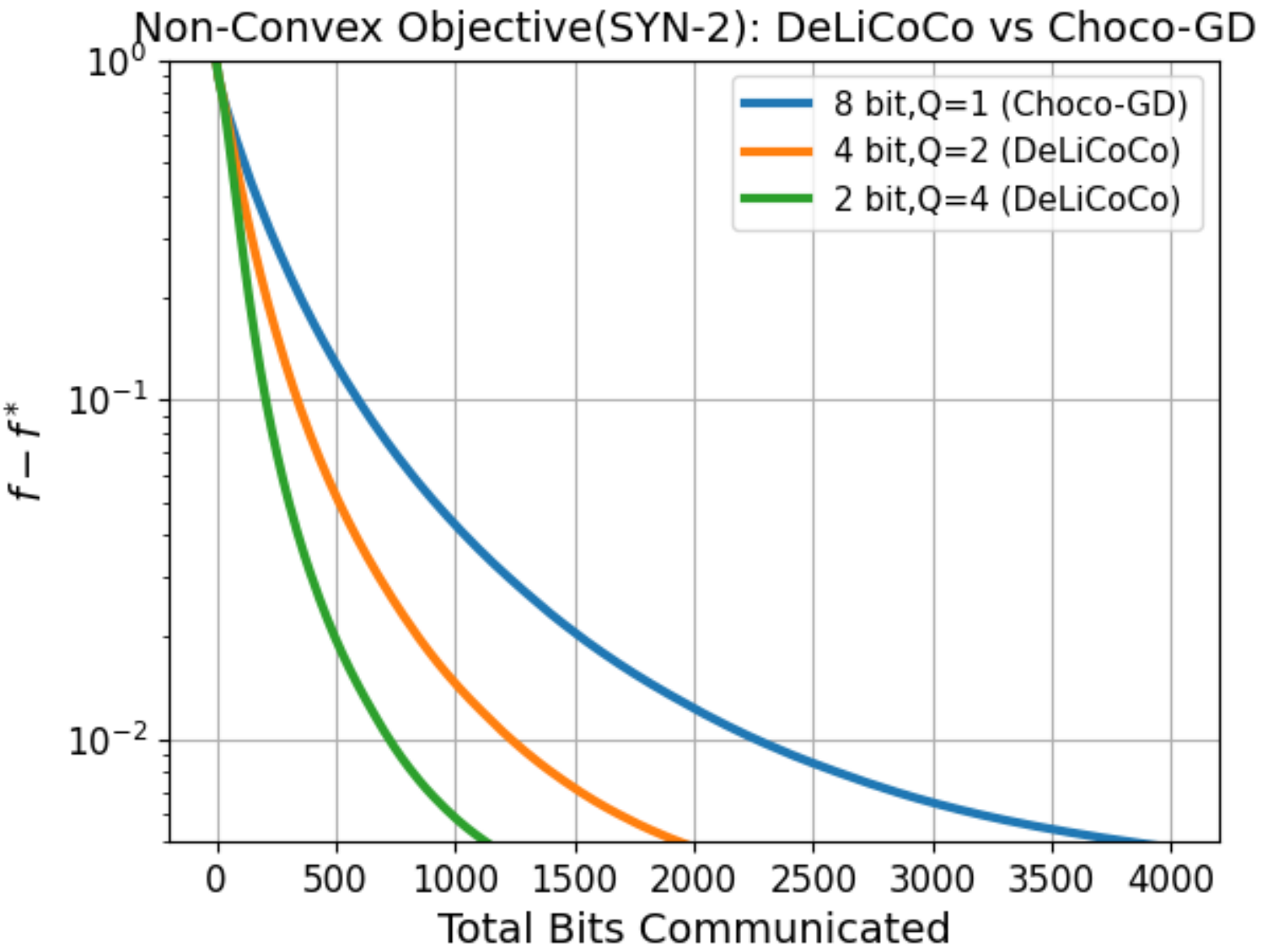}
  \end{center}
  \caption{\footnotesize Comparison of the proposed scheme and the benchmark \cite{koloskova2019decentralized} on the non-convex nonlinear regression task given a fixed communication budget per iteration.}
  \label{fig:motiv}
\end{figure}
\vspace{-3mm}
\subsection{Contribution}
{\color{black}We consider  decentralized nonconvex ML tasks in a communication-constrained settings. 
In such scenarios, the clients may need to compress their local updates (using, e.g., quantization and/or sparsification) before transmitting them to their neighbors. Our goal is to establish a communication-efficient decentralized scheme with accelerated convergence rates for nonconvex tasks. To this end,  we make an observation that in decentralized optimization tasks given a \textit{fixed communication budget per iteration}, having multiple consensus (aka gossiping) steps with aggressive
compression results in a smaller error in terms of the number of communicated bits compared to having just one consensus step with higher precision. As a concrete example
we consider a simple decentralized nonlinear regression task (see Section \ref{sec:exp}) and depict the training error versus the communicated bits in Figure \ref{fig:motiv}. As the figure shows increasing the number of  consensus steps (denoted by  $Q$) with a lower quantization precision requires fewer communicated bits to achieve a target accuracy.

This simple observation motivates us to theoretically study the effect of the number of gossiping steps on the rate of convergence in decentralized optimization in a communication-constrained setting. Specific contributions of this work can be summarized as follows:

1. We propose \underline{\textbf{De}}centralized \underline{\textbf{Li}}near Learning with \underline{\textbf{Co}}mmunication \underline{\textbf{Co}}mpression (DeLi-CoCo), an iterative  decentralized algorithm with arbitrary communication compression (both biased and unbiased compression operators) that performs multiple gossip steps in each iteration for faster convergence. 

2. By employing  $Q>1$ steps of compressed communication after each local gradient update, DeLi-CoCo achieves a linear rate of convergence to a near-optimal solution  for smooth nonconvex objectives satisfying the Polyak-\L{}ojasiewicz condition (see Theorem \ref{thm:2}). This rate matches the convergence rate of decentralized gradient descent (DGD) \cite{yuan2016convergence} with no communication compression. The proposed $Q$-step gossiping further helps to arbitrarily decrease the sub-optimality radius of the near-optimal solution, thereby improving upon the results of DGD \cite{yuan2016convergence} (see Corollary \ref{cor:1}).

3.  Our novel theoretical contributions enables us to demonstrate that given a fixed communication budget, increasing $Q$ and decreasing the precision of compression theoretically improves the convergence properties of DeLi-CoCo (see Section \ref{sec:dissc}). We verify our theoretical results via simple numerical experiments on both convex and nonconvex decentralized optimization tasks.}

\section{Significance and Related Work}
Decentralized learning and optimization have drawn significant attention in the past few years due to the increasing importance of privacy and high data communication costs of centralized methods \cite{koloskova2019decentralizeda,stich2018sparsified,tang2018communication,koloskova2019decentralized,shen2018towards,reisizadeh2019robust,he2018cola}. Decentralized topologies overcome the aforementioned challenges by allowing each client to exchange messages only with their neighbors without exchanging their local data, showing great potential in terms of scalability and privacy-preserving capabilities.

\subsection{Consensus with Compressed Communication}
The study of decentralized optimization problems dates back to 1980s \cite{tsitsiklis1984problems}. The main focus of early research in this area was on the task of average consensus where the goal of a network is to find the average of local variables (i.e., agents' model vectors) in a decentralized manner. Conditions for asymptotic and non-asymptotic convergence of the decentralized average consensus in a variety of settings including directed and undirected time-varying graphs have been established in the seminal works \cite{jadbabaie2003coordination,ren2005consensus,ren2007information,cai2012average,cai2014average,kempe2003gossip,xiao2004fast, boyd2006randomized}. Recently, a pioneering work \cite{koloskova2019decentralized} proposed the first communication-efficient 
average consensus/gossip algorithm that achieves a linear convergence rate and significantly improves the performance of existing quantized gossip methods \cite{carli2010quantized,carli2010gossip,fang2010distributed,li2010distributed}. In  \cite{koloskova2019decentralized} a stochastic decentralized algorithm only for strongly convex and smooth objectives is further developed. Such linearly convergent gossip methods have also recently been extended to the scenario where the communication graph of agents is directed and time-varying \cite{taheri2020quantized,chen2020dist}. 
In our work, we consider general nonconvex learning tasks and employ the proposed gossiping scheme of \cite{koloskova2019decentralized} 
as a subroutine. However, we  propose a new  decentralized algorithm  with arbitrary compression that leverages multiple gossip steps to collaboratively solve nonconvex problems under the Polyak-\L{}ojasiewicz condition \cite{polyak1963gradient,karimi2016linear}. 

\subsection{Distributed Optimization with Compressed Communication} Distributed optimization is one of the richest topics at the intersection of machine learning, signal processing and control. Consensus/gossip algorithms have enabled distributed optimization of (non)convex objectives (e.g., empirical risk minimization) by modeling the task of decentralized optimization as noisy consensus. Examples include the celebrated distributed (sub)gradient descent algorithms (DGD) \cite{nedic2009distributed, johansson2010randomized,yuan2016convergence}. These schemes consider small-scale problems where the clients can communicate uncompressed messages to their neighbors.
Designing communication-efficient distributed optimization algorithms is an active area of research motivated by the desire to reduce the communication burden of multi-core and parallel optimization of ML models \cite{konevcny2018randomized,stich2018sparsified,karimireddy2019error}. Majority of the existing works consider distributed optimization tasks with master-slave architectures where the compression of communication is accomplished by using methods based on sparsification or quantization of gradients \cite{seide20141,strom2015scalable,konevcny2018randomized,stich2018sparsified,bernstein2018signsgd,alistarh2017qsgd}. An example of such a setting is federated learning \cite{mcmahan2017communication,konevcny2016federated} which enables distributed learning of an ML model in a cloud while the training data remains distributed across a large number of clients. {\color{black}Recent federated learning schemes that promote communication efficiency either focus on compressing the size of the client-to-cloud messages or decreasing the number of communication rounds \cite{koloskova2019decentralizeda,reisizadeh2019fedpaq,sun2019communication,wang2018cooperative,basu2019qsparse,chen2018lag}. In contrast to that line of work, we consider a general decentralized learning scenario and exploit the error feedback mechanism of \cite{seide20141,strom2015scalable,stich2018sparsified} as part of our proposed scheme to enable arbitrary compression while maintaining a linear convergence rate. More importantly, our focus in this paper is the importance of organizing the communication resources. That is, given a fixed communication budget what is the best strategy for the number of consensus steps and the precision of compression in order to achieve a smaller error in terms of the number of communicated bits.} 

It is worth noting that unlike a majority of decentralized optimization schemes including those with uncompressed communication  that require strong convexity to a achieve linear rate, e.g. \cite{seaman2017optimal,yuan2016convergence,nedic2017achieving,shi2015extra} -- except the recent results in \cite{yi2019linear,tang2019distributed} with full communication --  we only assume the Polyak-\L{}ojasiewicz condition which enables us to analyze nonconvex learning tasks. To our knowledge, the proposed algorithm is the first scheme  to achieve linear convergence in the decentralized setting with compressed communication under the Polyak-\L{}ojasiewicz condition. In Table \ref{tb:comp}, we compare the proposed scheme with the related works.

\begin{table*}[t]
\small
\caption{\small Comparison of convergence rates of different decentralized optimization algorithms under smoothness. 
In the table, $\rho,\rho_1,\rho_2,\rho_3 \in (0,1)$ denote different rates of linear convergence that depend on problem's properties. 
$\alpha_1$ and $\alpha_2$ the radii of suboptimality and depend on network and function properties. 
}
\ra{1}
\centering
\begin{tabular*}{1\linewidth}{@{}cccc@{}}\toprule 
Algorithm &Convergence &Setting&Compression\\ \midrule
DGD \cite{nedic2009distributed} &$\O(1/T)$ &Strong convexity, full gradient&{\color{red}\xmark}\\\midrule
DGD \cite{yuan2016convergence} &$\O(\rho_1^T)+\alpha_1$ &Restricted strong convexity, full gradient&{\color{red}\xmark}\\\midrule
EXTRA, SSDA \cite{shi2015extra,seaman2017optimal} &$\O(\rho_2^T)$ &Strong convexity, full gradient&{\color{red}\xmark}\\\midrule
DIGing\cite{nedic2017achieving} &$\O(\rho_3^T)$ &Strong convexity, full gradient&{\color{red}\xmark}\\\midrule
\cite{koloskova2019decentralized,tang2019texttt,lu2020moniqua}&$\O(1/\sqrt{T})$&Smooth non-convex, stochastic gradient&{\color{black}\cmark}\\\midrule
{\bf This work} &$\O(\rho^T)+C\alpha^Q$ &PL condition, full gradient&{\color{black}\cmark}\\
\bottomrule
\end{tabular*}
\label{tb:comp}
\end{table*}

%% file: back.tex
\section{Preliminaries and Background}
In this section, we propose our linearly convergent decentralized learning algorithm. First, we briefly overview a few important concepts and definitions.

We consider the standard decentralized optimization setup \cite{nedic2009distributed} where $n$ clients, each having a local function $f_i(.)$, aim to collaboratively reach $\x^\ast \in \mathcal{X}^\ast \subset \R^d$, an optimizer of \eqref{eq:problem}. Problem \eqref{eq:problem} can be written equivalently as \cite{yuan2016convergence,nedic2009distributed,koloskova2019decentralized,shi2015extra}
\begin{equation}\label{eq:problem2}
\min_{\mathbf{x}_1 = \dots =\x_n} \left[F(X):=\sum_{i=1}^n f_i(\mathbf{x}_i)\right],
\end{equation}
where $\x_i\in \R^d$ is the vector collecting the local parameters of client $i$, and $X \in \mathbb{R}^{d\times n}$ is a matrix having $\x_i$ as its $i\ts{th}$ column. Therefore, the goal of the agents in the network is to achieve consensus such that $\x_i = \x^\ast$ for some  $\x^\ast \in \mathcal{X}^\ast$; in matrix notation, $X = X^\ast$, where all the columns of $X^\ast$ are equal to $\x^\ast$, i.e. $X^\ast = \x^\ast \mathbf{1}^\top$. 

To solve \eqref{eq:problem2}, each client can communicate only with its neighbors, where the communication in the network is modeled by a graph. Specifically, we assume each node $i$ associates a non-negative weight
$w_{ij}$ to any node $j$ in the network, and $w_{ij}> 0$ if and only if node $j$ can communicate with node $i$, and $w_{ii}> 0$ for all $i$. Let $W = [w_{ij}] \in [0,1]^{n \times n}$ be the matrix that collects these weights. We call $W$ the mixing or gossip matrix 
and state some its properties (following \cite{xiao2004fast}) below.
\begin{assumption}\label{assp:1}
{(\textbf{Mixing Matrix})}
\textit{The gossip matrix $W = [w_{ij}] \in [0,1]^{n \times n}$ associated with a connected graph is non-negative, symmetric and doubly stochastic, i.e.}
\begin{equation}
W = W^\top, \qquad W\mathbf{1} =  \mathbf{1}.
\end{equation}
\textit{Under this condition, eigenvalues of  $W$ can be shown to satisfy $1 = |\lambda_1(W)| > |\lambda_2(W)| \geq \dots \geq |\lambda_n(W)|$ \cite{xiao2004fast}. Furthermore, $\delta := 1 - |\lambda_2(W)| \in (0,1]$}
\textit{is the so-called spectral gap of $W$.}
\end{assumption}
A large spectral gap implies a faster convergence rate of decentralized algorithms. When  the graph is fully connected and $\text{deg}(i) = n$, with $W = \mathbf{1}\mathbf{1}^\top/n$, it holds that $\delta = 1$ which in turn implies consensus can be achieved exactly after one iteration of message passing. 

{\color{black}Designing the communication network and its associated mixing matrix $W$ with a large  spectral gap is an important task and an active area of research in multi-agent systems (see e.g. \cite{jadbabaie2003coordination,xiao2004fast,nedic2009distributed}) which is beyond the scope of this work. Here, we simply make the common assumption that $W$ and its spectral gap $\delta$ are known and can be used as inputs of our proposed algorithm.}

We now define some commonly assumed properties of the objective function.\footnote{In this work, $\|.\|$ denotes the $\ell_2$ norm for vectors and Frobenius norm for matrices, respectively.}
\begin{assumption}\hspace{-0.1 cm}{(\textbf{Smoothness})}
\textit{ Each local objective function is $L_i$-smooth, i.e., for all $\x,\y \in \R^d$
\begin{equation}
f_i(\x)\leq f_i(\y)+(\x-\y)^\top \nabla f_i(\y)+\frac{L_i}{2}\|\x-\y\|^2. 
\end{equation}
}
\end{assumption}
We will find it useful to define $L := \sum_i L_i/n$ and $\hat{L} := \max_i L$.
\begin{assumption}{(\textbf{Polyak-\L{}ojasiewicz Condition})}
\textit{The objective function satisfies the Polyak-\L{}ojasiewicz condition (PLC) with parameter $\mu$, i.e. for all $\forall \x \in \R^d$
\begin{equation}
\|\nabla f(\x)\|^2 \geq 2\mu (f(\x)-f^\ast), \quad \mu>0, \quad f^\ast = \min_\x f(\x).
\end{equation}
}
\end{assumption}
The Polyak-\L{}ojasiewicz condition implies that when multiple global optima exist, each stationary point of the objective function is a global optimum \cite{polyak1963gradient,karimi2016linear}. This setting enables studies of modern large-scale ML tasks such as training of deep neural networks that are generally nonconvex but are fairly likely to satisfy PLC \cite{liu2020toward}. It is worth noting that $\mu$-strongly convex functions satisfy PLC with parameter $\mu$ -- thus, PLC is a weaker assumption than strong convexity.

Convergence of centralized gradient descent under PLC follows a very simple analysis \cite{polyak1963gradient,karimi2016linear}. However, in decentralized settings with compression, analysis of the existing algorithms, e.g. \cite{yuan2016convergence,seaman2017optimal,koloskova2019decentralized}, relies on co-coercivity of strongly convex objectives (see Theorem 2.1.11 in \cite{nesterov2013introductory}).
Unfortunately, the results of such analysis do not generalize to PLC settings. In this paper, by performing a novel and simple convergence analysis, we establish convergence of DeLi-CoCo for decentralized nonconvex problems with compressed communication under PLC.

Finally, we characterize the compression operator $\mathcal{C}$ that we use in our algorithm. 
The following assumption is standard and has been previously made by \cite{stich2018sparsified,karimireddy2019error,koloskova2019decentralized}.
\begin{assumption}\label{assp:4}
(\textbf{Contraction Compression})
\textit{The compression operator $\mathcal{C}$ satisfies}
\begin{equation}
\E_{\mathcal{C}} \left[\|\mathcal{C}(\x)-\x\|^2  \text{ }|\text{ }\x\right]  \leq (1-\omega)\|\x\|^2, 
	\end{equation}  
\textit{for all $\x \in \R^d$ where $0<\omega\leq 1$ and the expectation is over the internal randomness of $\mathcal{C}$.}
\end{assumption}
Note that $\mathcal{C}$ can be a biased or an unbiased compression operator including:
\begin{itemize}
    \item Random selection of $k$ out of $d$ coordinates or $k$ coordinates with the largest magnitudes. In this case  $\omega = k/d$ \cite{stich2018sparsified}. We denote these two by $\mathrm{rand}({\omega})$ and $\mathrm{top}({\omega})$, respectively.
    \item Setting $\mathcal{C}(\x) = \x$ with probability $p$ and $\mathcal{C}(\x) = 0$ otherwise. In this case $\omega = p$ \cite{koloskova2019decentralized}. We denote this by $\mathrm{rand2}({\omega})$.
    \item $b$-bit random quantization (i.e., the number of quantization levels is $2^{b}$)
\begin{equation}
\mathrm{qsgd}_b(\x) = \frac{\mathrm{sign}(\x) \|\x\|}{2^{b}w}\left\lfloor2^{b}\frac{|\x|}{\|\x\|}+\mathbf{u}\right\rfloor,
\end{equation}
where $w = 1+\min\{\sqrt{d}/2^{b},d/2^{2b}\}$, $\mathbf{u} \sim [0,1]^d$, and $\mathrm{qsgd}_b(\mathbf{0}) = \mathbf{0}$. In this case, $\omega = 1/w$ \cite{alistarh2017qsgd}.
\end{itemize}
\section{Compressed Decentralized Learning}
In this section, we present our proposed algorithm for solving \eqref{eq:problem2} iteratively in a decentralized manner where the agents are restricted to communicate compressed information. In particular, we aim to develop a scheme that by relying on performing multiple low precision compressed gossiping steps achieves a smaller error in terms of the number of communicated bits.

The proposed DeLi-CoCo scheme (see Algorithm 1) consists of two main subroutines: (i)  update of the local variable $\x_i$ via gradient descent, and (ii) exchange of compressed messages between neighboring clients by performing $Q\geq 1$ compressed gossiping steps. 

Let $t = 1, \dots, T$ denote the $t\ts{th}$ iteration of DeLi-CoCo and let $q = 0,\dots.Q-1$ denote the $q\ts{th}$ compressed gossiping/consensus step. Each client $i$ maintains three local variables: $\x_{t,i}^{(q)}$, $\z_{t,i}^{(q)}$, and $\s_{t,i}^{(q)}$. Here, $\x_{t,i}^{(q)}$ denotes the vector of current local parameters of node $i$, while $\z_{t,i}^{(q)}$, and $\s_{t,i}^{(q)}$ are maintained locally to keep track of the compression noise and be used as an error feedback for subsequent iterations, respectively \cite{stich2018sparsified,koloskova2019decentralized}.
 
 At iteration $t$, each client updates its own local parameters by performing a simple gradient descent update according to
 \begin{equation}
  \x_{t,i}^{(0)} = \x_{t-1,i}^{(Q)} - \eta \nabla f_i(\x_{t-1,i}^{(Q)}),
 \end{equation}
 where $\eta>0$ is a constant learning rate 
 specified in Theorem \ref{thm:2}.
 Following the gradient update, we propose to perform $Q$ compressed gossiping steps in a decentralized manner to further update the local parameters as well as the error feedback variables. Intuitively,  this $Q$-step procedure is a crucial part of DeLi-CoCo  that enables updated parameters $\x_{t,i}^{(0)}$ to converge to their average value. 
 
To perform the $(q+1)\ts{st}$ gossiping step, each node generates the message
 $\mathcal{C}(\x_{t,i}^{(q)}-\z_{t,i}^{(q)})$, where  $\mathcal{C}: \R^d \rightarrow \R^d$ denotes the (potentially random) compression operator, and $\z_{t,i}^{(q)}$ is a parameter that keeps track of the compression error. The compressed message $\mathcal{C}(\x_{t,i}^{(q)}-\z_{t,i}^{(q)})$ is communicated to update $\s_{t,i}^{(q)}$, and then it is further used by the transmitting client as an error feedback to update $\z_{t,i}^{(q)}$: 
 \begin{equation}
 \begin{aligned}
  \s_{t,i}^{(q+1)}&= \s_{t,i}^{(q+1)}+\sum_{j=1}^n w_{ij}\mathcal{C}(\x_{t,j}^{(q)}-\z_{t,j}^{(q)}),\\ \z_{t,i}^{(q+1)} &=\z_{t,i}^{(q)}+\mathcal{C}(\x_{t,i}^{(q)}-\z_{t,i}^{(q)}).
   \end{aligned}
 \end{equation}

 Intuitively, $\z_{t,i}^{(q)}$ and the error feedback mechanism enable all the local information to be transmitted eventually with a delay that depends on  $\mathcal{C}$. 
 
 Then, to accomplish the $(q+1)\ts{st}$ gossiping step, each client performs
\begin{equation}
\x_{t,i}^{(q+1)} = \x_{t,i}^{(q)}+\gamma(\s_{t,i}^{(q+1)}-\z_{t,i}^{(q+1)}),
 \end{equation}
 where $0<\gamma\leq 1$ is the gossiping/consensus learning rate whose exact value will be specified in Theorem \ref{thm:2}.
 After performing compressed gossiping for $Q$ steps, the $t\ts{th}$ iteration of DeLi-CoCo is complete. 
 
 The above update rules are summarized in Algorithm 1 where we use an equivalent and useful matrix notation where $\x_{t,i}^{(q)}$, $\s_{t,i}^{(q)}$, and $\z_{t,i}^{(q)}$ are stored as the $i\ts{th}$ column of $X^{(q)}_t$, $S^{(q)}_t$, and $Z^{(q)}_t$, respectively.
\begin{remark}
Let $Q = 1$, $\gamma = 1$, and assume there is no compression, i.e. $\mathcal{C}(\x_{t,i}^{(q)}-\z_{t,i}^{(q)}) = \x_{t,i}^{(q)}-\z_{t,i}^{(q)}$. Then the proposed DeLi-CoCo scheme reduces to the DGD \cite{yuan2016convergence}. If $Q = 1$, $\eta = \O(1/T)$, and clients perform local stochastic gradient updates, the proposed scheme reduces to Choco-SGD \cite{koloskova2019decentralized}. We will show in Section 4 that by performing $Q> 1$ gossiping steps and reducing the precision of compression, DeLi-CoCO achieves a smaller error given a fixed communication budged.
 \end{remark}
 
{\color{black} 
\textbf{Practical Considerations}: In a scenario where there is negligible latency and synchronization among the nodes, the proposed algorithm that relies on multiple compressed gossiping steps achieves a faster convergence rate and further requires  fewer total number of bits needed for communication (see Section \ref{sec:exp}). Even if latency/synchronization constraints are taken into consideration the proposed algorithm leads to significant savings in the total number of communicated bits. However, in this case, certifying that our algorithm converges faster (with respect to wall-clock time) is difficult without knowing the actual time expended on synchronization and the latency of the communication structure, and is left for future work. Thus, in communication-constrained settings, Algorithm \ref{alg:C} with multiple gossiping steps ($Q>1$) is preferred.
}
\begin{algorithm}[t]
	\caption{The DeLi-CoCo Algorithm}
	\label{alg:C}
	\begin{algorithmic}[1]
		\STATE {\bfseries Input:} 
		stepsize $\eta$, consensus stepsize $\gamma$, number of gradient iterations $T$, number of consensus steps per gradient iteration $Q$, mixing matrix $W$; initialize $X_{0}^{(Q)}$, $Z_0^{(0)} = S_0^{(0)} = \mathbf{0}$.
		\FOR{ $t =1,\dots, T$}
		\STATE $X^{(0)}_t = X^{(Q)}_{t-1}-\eta \nabla F(X^{(Q)}_{t-1})$ (local gradient update)
		\FOR{$q =0, 1,\dots, Q-1$}
		\STATE $S_t^{(q+1)} =S_t^{(q)}+\mathcal{C}(X^{(q)}_t-Z_t^{(q)})W$ (Exchanging messages)
		\STATE $Z_t^{(q+1)} =Z_t^{(q)}+\mathcal{C}(X^{(q)}_t-Z_t^{(q)})$ (Compression error feedback)
		\STATE $X^{(q+1)}_t = X^{(q)}_t+\gamma(S_t^{(q+1)}-Z_t^{(q+1)})$ (Local gossip update)
		\ENDFOR
		\STATE $Z_{t+1}^{(0)} = Z_{t}^{(Q)}$, $S_{t+1}^{(0)} = S_{t}^{(Q)}$
		\ENDFOR
	\end{algorithmic}
\end{algorithm}

%% file: analysis.tex
In this section we analyze the convergence properties of DeLi-CoCo.
First, We define the following quantities $\Delta^2 :=\max_{\x^\ast \in \mathcal{X}^\ast}\sum_{i=1}^n \|\nabla f_i(\x^\ast)\|^2$ and $R_0 := F(X_{0}^{Q})-f^\ast$.
The main results of our convergence analysis are summarized in the following theorem. 
{\color{black}
\begin{theorem}\label{thm:2}
Suppose Assumptions \ref{assp:1}-\ref{assp:4} hold. Define 
\begin{equation}\label{eq:notationf}
\begin{aligned}
Q_0&:=\left\lceil{\log\left(\bar{\rho}/46\right)}\Bigg/{\log\left(1-\frac{\delta\gamma}{2}\right)}\right\rceil, \hspace{2mm} \bar{\rho} := 1-\frac{\mu}{n\hat{L}},\\
\gamma &= \frac{\delta\sqrt{\omega}}{16\delta+\delta^2-8\delta\sqrt{\omega}+(4+2\delta)\lambda_{\max}^2(I-W)}.
\end{aligned}
\end{equation}
Then, if the nodes are initialized such that $X_{0}^{(Q)} = \mathbf{0}$, for any $Q>Q_0$ after $T$ iterations the iterates of DeLi-CoCo with $\eta= \frac{1}{\hat{L}}$ satisfy
\begin{multline}\label{eq:bound:new:plc}
\E_{\mathcal{C}}[F(X^{(Q)}_{T})]-f^\ast =\O\biggl(\frac{\Delta^2}{1-\bar{\rho}}\left(1-\frac{\delta \gamma}{2}\right)^{\frac{Q}{2}} +\left[1+\frac{L}{\mu\bar{\rho}}\left(1+\left(1-\frac{\delta{\gamma}}{2}\right)^{\frac{Q}{2}}\right)\right]R_0\rho^T\biggr).
\end{multline}
\end{theorem}

\subsection{Discussion}\label{sec:dissc}
The result of Theorem \ref{thm:2} confirms our earlier empirical observation in Figure \ref{fig:motiv}. To see this, note that the convergence rate depends on $\left(1 - \frac{\delta \gamma}{2}\right)^{Q/2} < \left(1-\frac{\delta^2\sqrt{\omega}}{82}\right)^{\frac{Q}{2}}$ -- we shall analyze this upper bound to motivate the benefit of having higher $Q$. Consider two pairs of $(Q_1,\omega_1)$ and $(Q_1\times c,\omega_1/c)$ where $c>0$ is an integer that determines the allocation of communication resources, and $(Q_1,\omega_1)$ satisfies the conditions stated in Theorem \ref{thm:2}. The proposed scheme for both of these pairs require the same amount of communication budget. Upon defining $ g(c):=\left(1-\frac{\delta^2\sqrt{ \omega}}{82\sqrt{c}}\right)^{\frac{c\times Q}{2}}$, in Figure \ref{fig:gc} we depict the value of $g(c)$ versus $c$ for various values of spectral gap $\delta$. As the figure shows $g(c)$ is decreasing in $c$ meaning that for a fixed communication budget, increasing the number of gossiping steps $Q$ and decreasing the compression parameter $\omega$ theoretically results in improved convergence properties as both terms in \eqref{eq:bound:new:plc} incur smaller values. Intuitively, this is expected since using the Bernoulli inequality $\left(1-\frac{\delta^2\sqrt{\omega_1}}{82\sqrt{c}}\right)^
{\frac{c Q_1}{2}}
\approx \left(1-\frac{\delta^2\sqrt{c\omega_1}}{82}\right)^
\frac{Q_1}{2}$, 
which is much smaller than $\left(1-\frac{\delta^2\sqrt{\omega_1}}{82}\right)^{\frac{Q_1}{2}}$. 
This theoretical result is indeed consistent with our earlier empirical observation in Figure \ref{fig:motiv} thereby showing the advantage of Algorithm \ref{alg:C} that advocates the use of multiple gossiping steps to achieve a smaller error in terms of the number of communicated bits.
\begin{figure}[t]
  \begin{center}
    \includegraphics[width=0.5\textwidth]{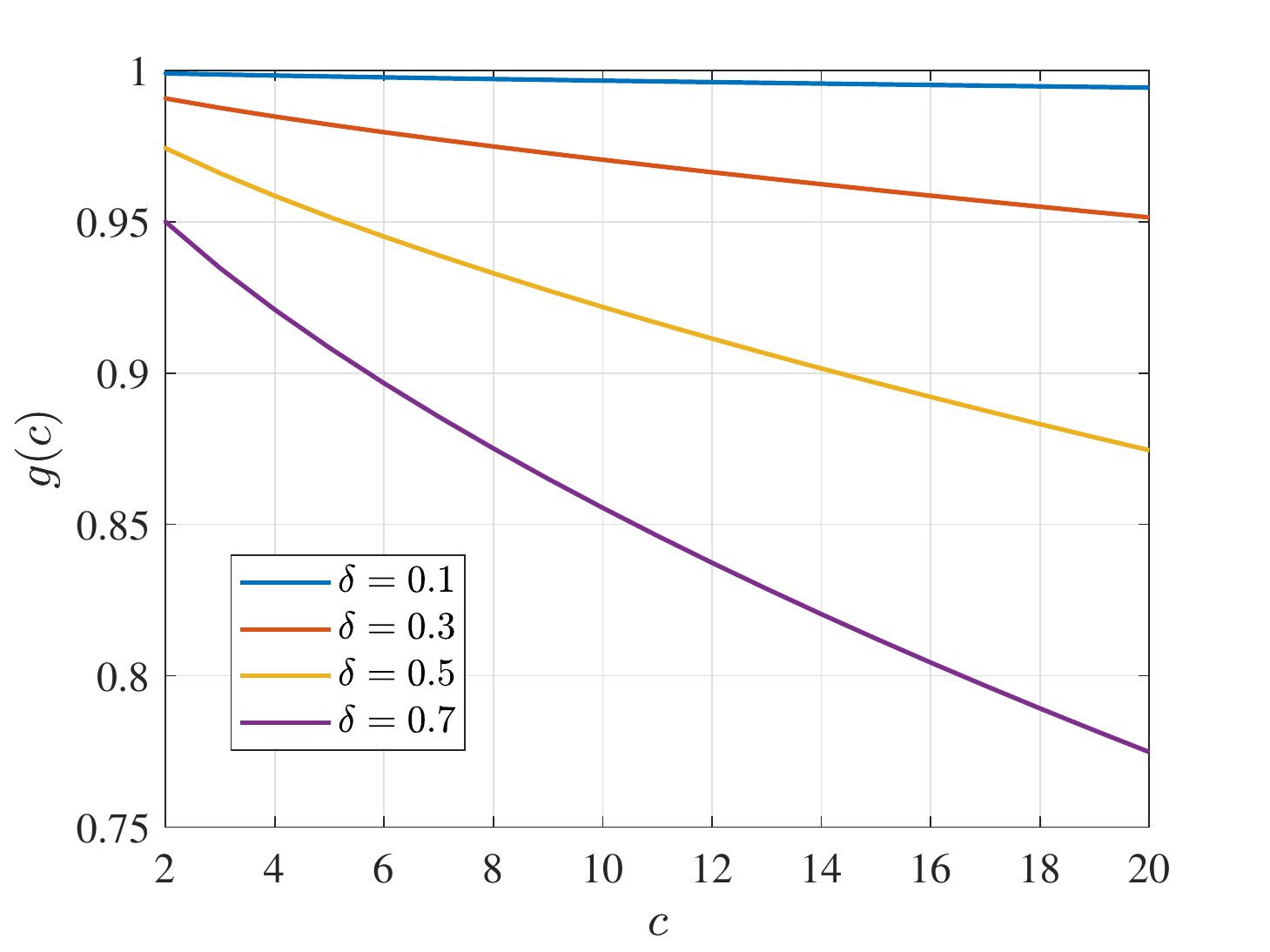}
  \end{center}
  \caption{\footnotesize 
  Variation of $g(c)$ vs. $c$ for different values of $\delta$.}
  \label{fig:gc}
\end{figure}
}

We further highlight the following remarks:

{\bf 1.  Comparison to DGD:}
We compare our result to the prior work in \cite{rogozin2019projected,yuan2016convergence} that assume exact communication. First, in contrast to \cite{yuan2016convergence,rogozin2019projected}, our analysis is carried out under PLC without assuming (restricted) strong convexity. 
The radius of the near-optimal neighborhood in \cite{yuan2016convergence} (see Theorem 4 there) is proportional to $\Delta/\delta$ while in our case, by using the proposed $Q$- step compressed gossiping procedure, the radius is proportional to $\Delta^2(1-\delta)^{\frac{Q}{2}}$; 
in fact, we can make the bound arbitrarily small by performing a sufficiently large number of gossiping steps $Q$ (see Corollary \ref{cor:1}).

{\bf 2. Effect of Compression:} Our results reveal that compression of messages using contraction operators can be thought of as weakening the connectivity property of the communication graph by inducing spectral gap $\delta' = \delta \omega$.  As $\omega$ approaches zero, the consensus learning rate decreases. Hence, as per intuition, a larger $Q$ is required to satisfy the conditions in the statement of Theorem \ref{thm:2}. 

{\bf 3. Almost Linear Convergence:}
Our analysis  further reveals that at the cost of increased number of rounds of communication, the suboptimality radius can be arbitrarily reduced. In particular, $\E_{\mathcal{C}}[F(X^{(Q)}_{t})]-f^\ast \leq  \epsilon$ accuracy can be achieved after $\O(\log^2(1/\epsilon))$ rounds of communication by setting $Q = T = \log(1/\epsilon)$. However, in practice it suffices to use a small $Q$ to achieve a competitive performance compared to centralized and decentralized schemes with no compression.

{\bf 4.  Power of Overparameterization:} 
Consider the case that \eqref{eq:problem}  corresponds to a decentralized regression or classification task wherein the model architecture is expressive
enough to completely fit or {\em  interpolate} the training data distributed among the clients \cite{ma2017power,vaswani2018fast,vaswani2019painless}, e.g. in the case of over-parameterized neural networks or functions satisfying a certain growth condition \cite{schmidt2013fast,cevher2019linear}. 
Then any stationary point of $f$ will also be a stationary point of each of the $f_i$'s and thus $\Delta^2=0$. Therefore, in the this setting and under PLC, Deli-CoCo  converges exactly at a linear rate of $\O(\log(1/\epsilon))$ by setting $Q$ to be a constant independent of $\epsilon$. 
\begin{corollary}\label{cor:1}
Instate the notation and hypotheses of Theorem \ref{thm:2}. 
In order to achieve $\E_{\mathcal{C}}[F(X^{(Q)}_{T})]-f^\ast \leq \epsilon$, Deli-CoCo requires $\tau = \O(\log^2(1/\epsilon))$ rounds of communication if $\Delta\neq 0$, and $\tau = \O(\log(1/\epsilon))$ if $\Delta=0$.
\end{corollary}
\textit{To our knowledge, DeLi-CoCo is the first algorithm attaining a linear convergence rate for decentralized nonconvex optimization with compressed communication in the interpolation regime}. Notice that linear convergence even in the centralized setting necessitates $T = \O(\log 1/\epsilon)$. In the decentralized setting under strong convexity (SC), without using techniques such as gradient tracking \cite{shi2015extra}, DGD based schemes use either $\eta = \O(\frac{1}{t})$ to have $\O(1/\epsilon)$ rounds of communication (e.g. Choco-GD or DGD \cite{koloskova2019decentralized,nedic2009distributed}), or a fixed stepsize (independent of $T$) to achieve linear convergence to a \textit{near-optimal} solution \cite{yuan2016convergence}. Corollary \ref{cor:1} states without over-parameterization $Q = \O(\log 1/\epsilon)$ enables our algorithm to converge to an $\epsilon$-accurate solution under PLC with $\O(\log^2 1/\epsilon)$ rounds of communication, which is a significant improvement over $\O(1/\epsilon)$ for DGD based schemes with decaying step-size.

{\bf 5. Implications for Federated Learning:} Theorem \ref{thm:2} also implies a near linear convergence rate for federated learning tasks -- in their simplest form -- satisfying PLC under compressed communication. This scenario corresponds to  a decentralized learning problem over a network  with $\delta = 1$ \cite{wang2019matcha} under which DeLi-CoCo efficiently delivers a stationary solution.
Nonetheless, 
there are some open problems and issues such as delayed communication and intermittent client availability in federated learning \cite{kairouz2019advances} that are not considered here. 

{\bf 6.  Results under strong convexity:}
Since PLC is implied by strong convexity, Theorem \ref{thm:2} provides a convergence rate for strongly convex and smooth objectives. We make this more explicit in \Cref{thm:1}: 
{\color{black}
\begin{theorem}\label{thm:1}
Suppose Assumptions 1,2, and 4 hold. Further, assume each $f_i$ is strongly convex with parameter $\mu_i$, and define $\mu = \sum_i \mu_i/n$, $\hat{\mu} = \min_i \mu_i$ and $D_0 := \|X^{(Q)}_0-X^\ast\|^2$.
Define 
\begin{equation}\label{eq:notationc}
\begin{aligned}
Q_0&:=\left\lceil{\log\left(\ell/46\right)}\Bigg/{\log\left(1-\frac{\delta\gamma}{2}\right)}\right\rceil,\hspace{2mm}
\ell := 1-\frac{\hat{\mu}}{\hat{L}},\\
\gamma &= \frac{\delta\sqrt{\omega}}{16\delta+\delta^2-8\delta\sqrt{\omega}+(4+2\delta)\lambda_{\max}^2(I-W)}.
\end{aligned}
\end{equation}
Then, if the nodes are initialized such that $X_{0}^{(Q)} = \mathbf{0}$, for any $Q>Q_0$ after $T$ iterations the iterates of DeLi-CoCo with $\eta= \frac{1}{\hat{L}}$ satisfy
\begin{equation}
\begin{aligned}
    \E_{\mathcal{C}}\|X^{(Q)}_{T}-X^\ast\|^2= \O\biggl(\frac{\Delta^2}{\hat{\mu}^2}\left(1-\frac{\delta\gamma}{2}\right)^{Q} +\left[1+\frac{T}{\ell^2 (\hat{L}-\hat{\mu})}\left(1-\frac{\delta\gamma}{2}\right)^{Q}\right]D_0 \ell^T\biggr).
    \end{aligned}
\end{equation}
\end{theorem}
}
\vspace{-6mm}
\subsection{Proof Outline}
Here, we briefly discuss the main ideas of the proof. Further details are in the supplementary material. 

\textit{Perturbed iterate analysis:} Our proof relies on analyzing the (virtual) average iterates \cite{rogozin2019projected,yuan2016convergence}
\begin{equation}
\begin{aligned}
&\bar{X}_{t+1} = \bar{X}_{t} - \eta\bar{\nabla} F(\bar{X}_{t}), \quad \bar{X}_{t} = [\bar{\x}_{t},\dots,\bar{\x}_{t}], \\&\bar{\nabla} F(\bar{X}_{t}) := \frac{1}{n}[\nabla f(\bar{\x}_{t}), \dots, \nabla f(\bar{\x}_{t})].
\end{aligned}
\end{equation}
In particular, we  show that the average iterates converge linearly under PLC, i.e., $F(\bar{X}_{t})-f^\ast = \O(\bar{\rho}^t)$
The iterates of DeLi-CoCo can be thought of as (random) perturbations of the virtual sequence $\{\bar{X}_{t}\}$: $X^{(q)}_t = \bar{X}_t+E^{(q)}_t $ for some (random) error matrix $E^{(q)}_t  \in \mathbb{R}^{d\times n}$. Evidently, we need  to derive an upper bound on the error terms $E^{(q)}_t$.

\textit{Bounding the error:} We first bound the error term of the $t\ts{th}$ iteration according to
\begin{equation}
	\E_{\mathcal{C}}\|E^{(Q)}_t\|^2 \leq e_t^2 := \E_{\mathcal{C}}\|X^{(Q)}_t - \bar{X}_{t}\|^2+\E_{\mathcal{C}}\|X^{(Q)}_t - Z^{(Q)}_t\|^2.
\end{equation} 
The starting point of analyzing $e_t^2$ is to leverage the linear convergence of the average iterates $\{\bar{X}_t\}$ and the gossiping steps with error feedback \cite{koloskova2019decentralized}. However, as stated in the proof of Theorem 2 in \cite{koloskova2019decentralized}, the choice of gossiping stepsize $\gamma$ there is very sub-optimal which in turn results in a very conservative bound. To remedy this we leverage a tighter approximation thereby improving the dependence of $\gamma$ and $\omega$. Establishing this last result we argue that the sequence $e_t^2$ and in turn the error term  also  converge linearly. Finally, leveraging this result, we establish the claim of the Theorems. 

%% file: new_experiments.tex
We verify our theory on three simple yet common machine learning problems - logistic regression, linear regression and non-linear regression. 
{\color{black} We consider the top($k$),$\textrm{qsgd}$, and rand($\omega$) compression schemes and 
the ring and torus topologies to represent the communication graph of the network (see, e.g. \cite{xiao2004fast,koloskova2019decentralized}). 
}Before describing our experimental set-up, we describe the datasets.

Let $\{s^{(i)}_{1},\ldots,s^{(i)}_{n_{i}}\}$ denote the samples being processed in the $i\ts{th}$ node where $n_{i}$ is the total number of samples in the $i\ts{th}$ node.
Then, $f_{i}(\x) = \frac{1}{n_{i}} \sum_{j=1}^{n_{i}} \ell (\x,s^{(i)}_{j})$, where $\ell(.)$ denotes the loss function of the regression tasks that we explain next.

\textit{Logistic Regression:}
We use a binary version of MNIST \cite{lecun1998gradient} where the first five classes are treated as class 0 and the rest as class 1. We train a classifier with the binary cross-entropy loss. 
We consider a decentralized setting where the data is evenly distributed among all the nodes in a challenging sorted setting (sorted based on labels) where at most one node  acquires examples from both classes. 

\textit{Linear Regression:}
We train a linear regression model on $m=10000$ synthetic data samples $\{(\mathbf{a}_{i},y_{i})\}_{i=1}^{m}$ generated according to $y_{i} = \langle \boldsymbol{\theta}^\ast, \mathbf{a}_{i} \rangle + e_{i}$, where $\boldsymbol{\theta}^\ast \in \mathbb{R}^{2000}$, the $i\ts{th}$ input $\mathbf{a}_{i} \sim \mathcal{N}(\mathbf{{0}}, I_{2000})$, and noise $e_{i} \sim \mathcal{N}(0,0.05)$. We refer to this 
dataset as SYN-1. Here, we use the squared loss function with $\ell_{2}$-regularization.

\textit{Non-Convex Non-Linear Regression:}
We train a non-linear regression model on $m=10000$ synthetic data samples $\{(\mathbf{a}_{i},y_{i})\}_{i=1}^{m}$ generated as $y_i = \text{relu}(\langle \boldsymbol{\theta}^\ast, \mathbf{a}_i \rangle) + e_i$, where $\boldsymbol{\theta}^\ast \in \mathbb{R}^{2000}$, the $i\ts{th}$ input $\mathbf{a}_i \sim \mathcal{N}(\vec{{0}}, I_{2000})$, $e_i \sim \mathcal{N}(0,0.05)$ and $\text{relu}(z) = \text{max}(z,0)$ (i.e. the standard ReLU function). 
We call this synthetic dataset SYN-2 henceforth. We model this task as training a one-layer neural network having ReLU activation with the squared loss function and {$\ell_{2}$-regularization.

{
\color{black}
\textbf{Experimental Setting -- Fixed Communication Budget Per Iteration.}
In order to illustrate the value of having more gossiping steps (i.e. larger $Q$), we consider a simple setting where our communication budget in every iteration (involving one gradient computation step and $Q$ gossiping steps) is fixed. So for top-$k$/rand-$k$, we keep $Q k$ constant, whereas for $b$-bit qsgd, $Q b$ is kept constant. Since $Q k$/$Q b$ is kept constant, the total number of bits communicated will be proportional to the number of iterations $T$ (which is the  horizontal axis of the plots in Figures \ref{fig:1} and \ref{fig:2}).
In \Cref{fig:1}, we plot the training loss on the vertical axis (in log-scale) vs. the number of bits (order wise) on the horizontal axis for SYN-2 (non-convex non-linear regression) with qsgd. We maintain $Qb = 8$ and consider 3 different consensus learning rates $\gamma = \{0.05,0.1,0.15\}$ (keeping everything else the same).

In \Cref{fig:2}, we show similar plots for SYN-1 (strongly convex linear regression task) with top-$k$. Let $\omega = (k/d)*100$ ($d$ being the dimension of the vectors). We keep $Q\omega = 100$ (note that this is the same as maintaining $Q k$ constant) and consider 3 different values of the number of nodes $n = \{9,16,25\}$ (keeping everything else the same).
}

\textbf{Significance of large Q.}
In both \Cref{fig:1} and \Cref{fig:2}, observe that higher $Q$ at the expense of more aggressive compression leads to fewer gross total number of bits communicated -- as predicted by the results established in the beginning of  Section \ref{sec:dissc}. Note that if latency/synchronization time between the nodes is negligible, then having higher $Q$ also leads to faster convergence (since the total number of bits is proportional to $T$ in our setting). Further, in \Cref{fig:2} (for SYN-1, which has a strongly convex objective), observe that higher $Q$ results in almost straight line curves (recall that the training loss is plotted in log-scale) -- implying linear convergence. Further note that for $Q=1$ the curves are not straight lines. This verifies our theoretical findings in Corollary \ref{cor:1}. 


In \Cref{fig:3}, we show results for the logistic regression task on MNIST with rand($\omega$). \footnote{The results with rand($\omega$) and rand2($\omega$) were nearly identical and so we have only shown the plots for rand($\omega$) here.}  
Let $\omega = (k/d)*100$ ($d$ being the dimension of the vectors). We keep $Q\omega = 100$ (note that this is the same as maintaining $Q k$ constant) and consider the two most commonly used topologies, ring and torus with $n=9$. 

 Consistent with the results in \Cref{fig:1} and \Cref{fig:2}, observe that in \Cref{fig:3}, using a higher $Q$ at the expense of more aggressive compression leads to fewer gross total number of bits communicated -- for both ring and torus topologies. Also if latency/synchronization time between the nodes is negligible, then having higher $Q$ also leads to faster convergence (since the total number of bits is proportional to $T$ in our setting). 
\begin{figure*}[t]
\centering 
\subfloat[Ring]{
    \label{fig:3_a}
	\includegraphics[width=0.47\textwidth]{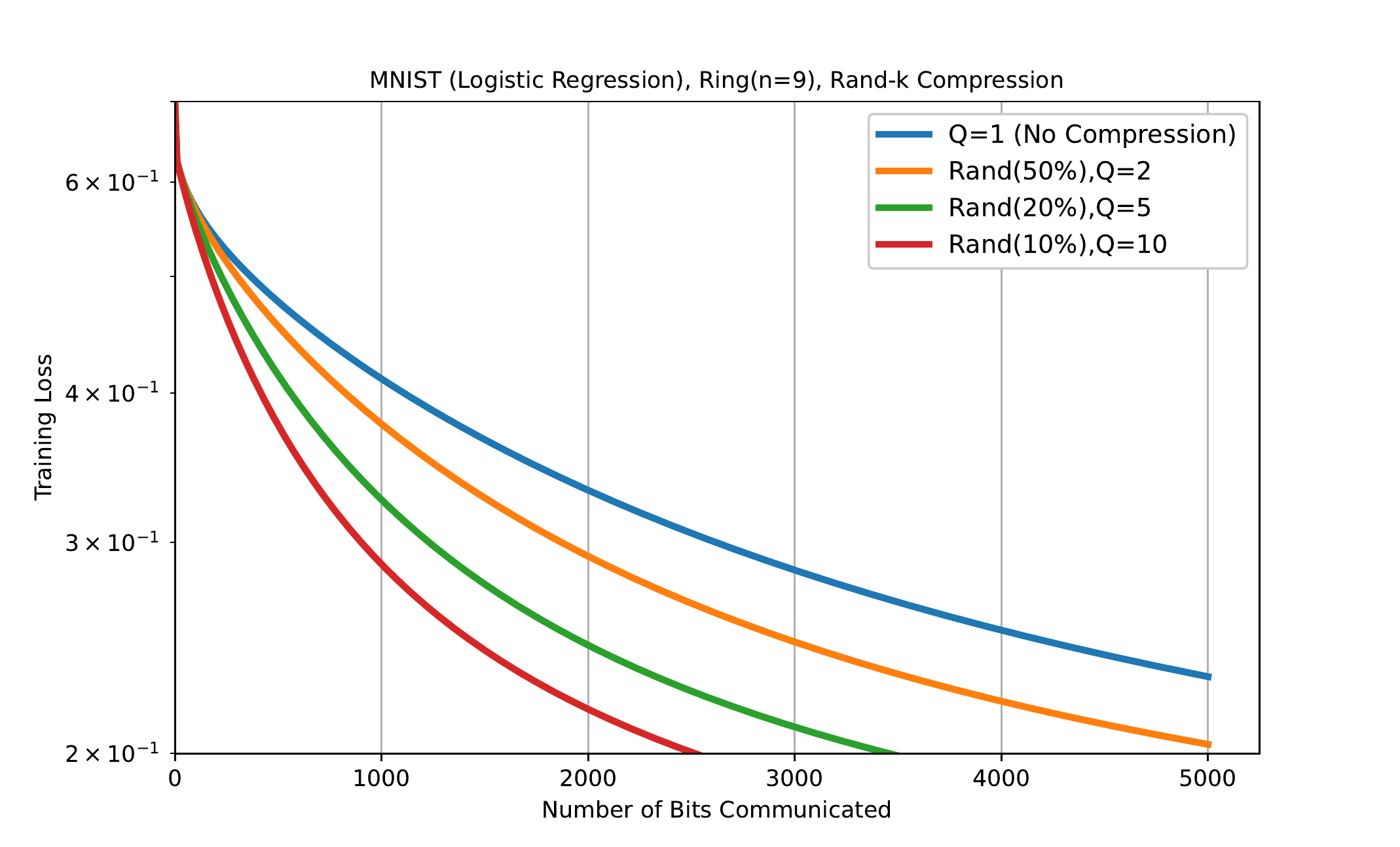}
	} 
\hspace{-0.3cm}
\subfloat[Torus]{
    \label{fig:3_b}
	\includegraphics[width=0.49\textwidth]{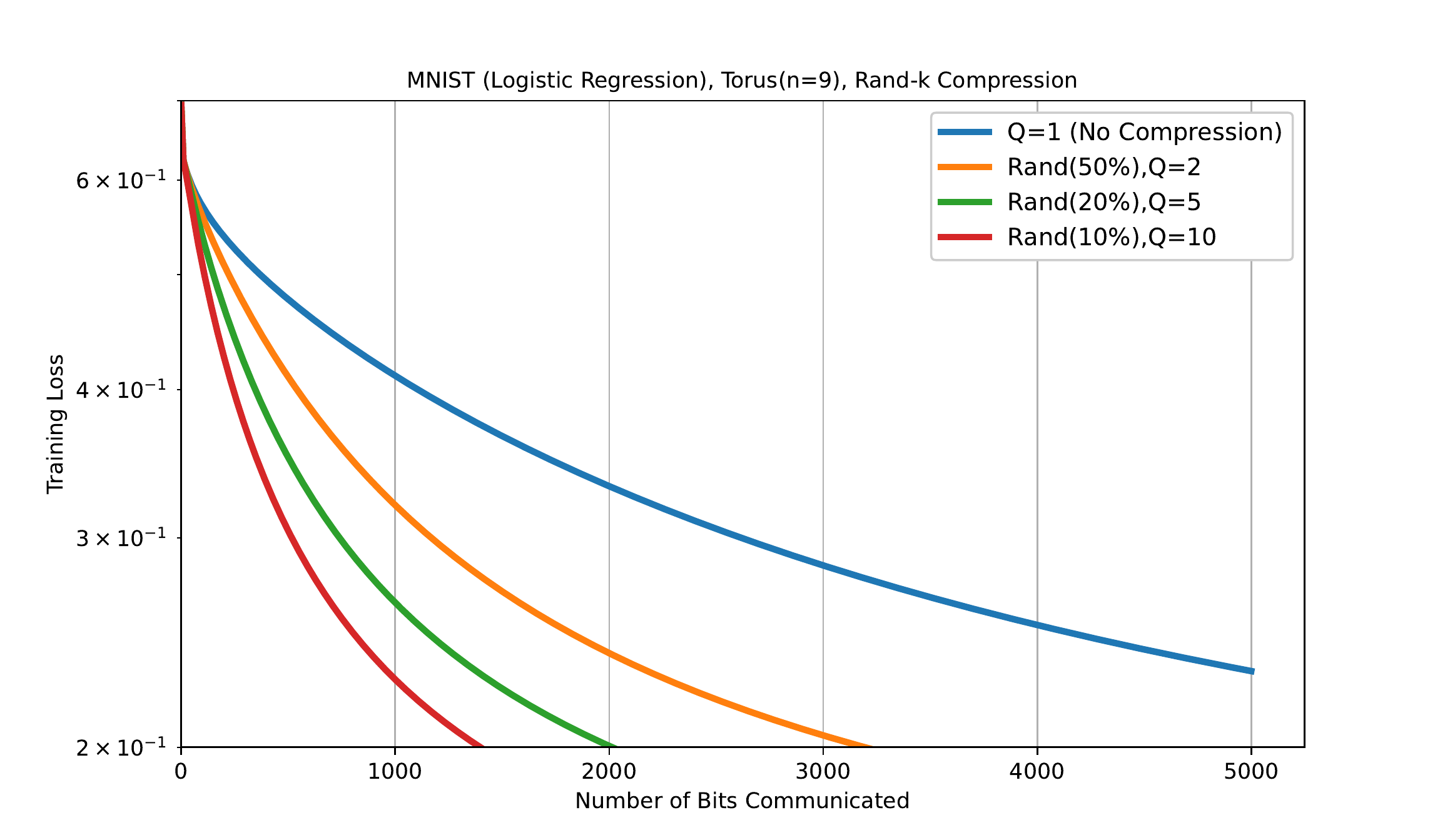}
	} 
\caption{\footnotesize Effect of different $(Q,\omega)$ pairs (where $\omega$ denotes the percentage of random co-ordinates picked in the rand quantization) such that $Q \omega = 100$, on the total number of bits communicated for MNIST logistic regression task. We consider the ring and torus topology with $n = 9$.
In both plots, $\gamma = 0.05$, $\ell_2$ regularization value = 0.001, and $\eta = 0.2$.}
\label{fig:3}
\end{figure*}

%% file: proof.tex
First, we introduce notation and state a few useful facts. Let $\mathcal{L}$ be a {\it linear subspace} of $d\times n$ matrices having identical columns with the projection operator $\mathcal{P}_{\mathcal{L}}(.)$ such that for all $X = [\x_1,\dots,\x_n]$ we have $\mathcal{P}_{\mathcal{L}}(X) = [\bar{\x},\dots,\bar{\x}]$, where $\bar{\x} = \frac{1}{n}\sum_i \x_i$. 

It then becomes evident that the update rule of the average iterate can be written equivalently as
\begin{equation}
	 \begin{aligned}
    \bar{X}_{t+1} &= \mathcal{P}_{\mathcal{L}}(\bar{X}_{t} - \eta\nabla F(\bar{X}_{t}))\\
    & = \bar{X}_{t} - \eta\mathcal{P}_{\mathcal{L}}(\nabla F(\bar{X}_{t}))\\
    \end{aligned}
\end{equation}
where $\mathcal{P}_{\mathcal{L}}(\nabla F(\bar{X}_{t}))=\frac{1}{n}[\nabla f(\bar{\x}_t),\dots, \nabla f(\bar{\x}_t)]$. While the second line does not hold in general, it does in our case due to the definition of $\mathcal{P}_{\mathcal{L}}$ and the fact that $\bar{X}_{t}\in \mathcal{L}$. Note that due to the defined projection operator, for $X^\ast = [\x^\ast,\dots,\x^\ast]$ where $\x^\ast$ is any of the global optima, it holds that $X^\ast \in \mathcal{L}$ and $ \mathcal{P}_{\mathcal{L}}(\nabla F(X^\ast)) = \mathbf{0}\in \mathcal{L}$.

Next, recall the non-expansiveness of projection (see Lemma 2.2.7 and  Corollary 2.2.3 in \cite{nesterov2013introductory}),
\begin{equation}\label{eq:nexpproj}
    \|\mathcal{P}_{\mathcal{L}}(X)-\mathcal{P}_{\mathcal{L}}(Y)\|\leq \|X-Y\|.
\end{equation}

Finally, we explore smoothness and strong convexity of $F(\cdot)$ (in the convex case). Since $\bar{X}_{i}\in \mathcal{L}$ for all  $i\in [T]$ and $\bar{X}^\ast = X^\ast \in \mathcal{L}$, by strong convexity and smoothness of $f(\cdot)$ it holds that
	\begin{equation}
	F(\bar{X}_{i})\leq F(\bar{X}_{j})+\langle\bar{X}_{i}-\bar{X}_{j}, \nabla F(\bar{X}_{j})\rangle+\frac{L}{2}\|\bar{X}_{i}-\bar{X}_{j}\|^2
	\end{equation}
	\begin{equation}
	F(\bar{X}_{i})\geq F(\bar{X}_{j})+\langle\bar{X}_{i}-\bar{X}_{j}, \nabla F(\bar{X}_{j})\rangle+\frac{\mu}{2}\|\bar{X}_{i}-\bar{X}_{j}\|^2
	\end{equation}
	 since $\|\bar{X}_{i}-\bar{X}_{j}\|^2 = n \|\bar{\x}_{i}-\bar{\x}_{j}\|^2$. Therefore, $F(\cdot)$ is $\mu$-strongly convex (in the convex scenario) and $L$-smooth on $\mathcal{L}$ and $\hat{\mu}$-strongly convex (in the convex scenario) and $\hat{L}$-smooth everywhere else.
	 
	 \begin{remark}
	 In the following proofs, for simplicity we make the simplifying assumption that the clients initialize their parameters such that $F(X_{0}^{(Q)}) = F(\bar{X}_{0})$. This can hold easily by setting all initial vectors to be equal to a vector $\x_0\in \R^d$.
	 \end{remark}
\subsection*{A.1. Proof of Convergence Under PLC}
As stated in the main paper, the proof relies on a perturbation analysis and interpreting the iterates of the proposed scheme as random perturbations of a virtual sequence $\bar{X}_{t}$ having identical columns. Then, leveraging linear convergence of the gossip subroutine and the fact $F(\bar{X}_{t})-f^\ast$ decreases linearly, we can show linear convergence of the expected difference $\E_{\mathcal{C}}\|\bar{X}_{t}-X_t^{(Q)}\|$ and, in turn, the error of the proposed scheme. Recall that
\begin{equation}
f^\ast:=\min_{\mathbf{x}} \left[f(\mathbf{x}):=\sum_{i=1}^n f_i(\mathbf{x})\right].
\end{equation}

Lemma \ref{lem:plc1} establishes linear convergence of the virtual average sequence in terms of the suboptimality of the function values. 
\begin{lemma}\label{lem:plc1}
	Let $\bar{X}_{t+1} = \bar{X}_{t} - \eta\mathcal{P}_{\mathcal{L}}(\nabla F(\bar{X}_{t}))$ denote a sequence such that $\bar{X}_{t} = [\bar{\x}_{t},\dots,\bar{\x}_{t}]$, and let $\eta\leq 1/L$. Then
	\begin{equation}
	F(\bar{X}_{t}) - f^\ast \leq [F(\bar{X}_{0})-f^\ast]\left(1-2\frac{\mu}{n}\eta(1-\frac{L\eta}{2})\right)^t.
	\end{equation}
\end{lemma}
\begin{proof}
First,  the equivalent update rule for $\bar{X}_t$, i.e.
\begin{equation}
   \bar{X}_{t+1} =  \arg\min_{\bar{Y}\in\mathcal{L}}\{g(\bar{Y}):=F(\bar{X}_{t})+\langle\nabla F(\bar{X}_{t}),\bar{Y}-\bar{X}_{t}\rangle+\frac{1}{2\eta}\|\bar{Y}-\bar{X}_{t}\|^2\}.
\end{equation}
to establish a useful smoothness result.  

From the optimally condition for the convex function $g(\bar{Y})$ (see Theorem 2.2.9 in \cite{nesterov2013introductory}) and the fact that $\bar{X}_{t+1}=\bar{X}_{t} - \eta\mathcal{P}_{\mathcal{L}}(\nabla F(\bar{X}_{t}))$, it follows that
\begin{equation}
    \langle\nabla F(\bar{X}_{t})-\mathcal{P}_{\mathcal{L}}(\nabla F(\bar{X}_{t})),\bar{Y}-\bar{X}_{t+1}\rangle\geq 0
\end{equation}
for any $\bar{Y}\in\mathcal{L}$. Now, using this result and choosing $\bar{Y}= \bar{X}_{t}\in\mathcal{L} $ we have
\begin{equation}
    \begin{aligned}
    F(\bar{X}_{t}) &\geq F(\bar{X}_{t})+\langle\nabla F(\bar{X}_{t})-\mathcal{P}_{\mathcal{L}}(\nabla F(\bar{X}_{t})),\bar{X}_{t+1}-\bar{X}_{t}\rangle\\
    & \geq g(\bar{X}_{t+1}) -\frac{1}{2\eta}\|\bar{X}_{t+1}-\bar{X}_{t}\|^2 -\langle\mathcal{P}_{\mathcal{L}}(\nabla F(\bar{X}_{t})),\bar{X}_{t+1}-\bar{X}_{t}\rangle\\
    &=g(\bar{X}_{t+1}) +\frac{\eta}{2}\|\mathcal{P}_{\mathcal{L}}(\nabla F(\bar{X}_{t}))\|^2,
        \end{aligned}
\end{equation}
where we used the definition of $g(\cdot)$ and the update rule $\bar{X}_{t+1}=\bar{X}_{t} - \eta\mathcal{P}_{\mathcal{L}}(\nabla F(\bar{X}_{t}))$. Now, $L$-smoothness of $F(\cdot)$ on $\mathcal{L}$ and the fact that $\eta\leq 1/L$ imply
    \begin{equation}
    \begin{aligned}
    F(\bar{X}_{t})&\geq F(\bar{X}_{t+1})+(\frac{1}{2\eta}-\frac{L}{2})\|\bar{X}_{t+1}-\bar{X}_{t}\|^2+\frac{\eta}{2}\|\mathcal{P}_{\mathcal{L}}(\nabla F(\bar{X}_{t}))\|^2\\
    &=F(\bar{X}_{t+1})+\eta(1-\frac{\eta L}{2})\|\mathcal{P}_{\mathcal{L}}(\nabla F(\bar{X}_{t}))\|^2.
    \end{aligned}
\end{equation}
Finally, using the definition of $\mathcal{P}_{\mathcal{L}}(\nabla F(\bar{X}_{t}))$, we can relate this last result to the gradient of $f$ at $\bar{x}_t$ to obtain
	\begin{equation}\label{eq:sm1}
	F(\bar{X}_{t+1})\leq F(\bar{X}_{t})-\frac{\eta}{n}(1-\frac{L\eta}{2})\|\nabla f(\bar{\x}_{t})\|^2.
	\end{equation}
Having established the above {\it smoothness} result and recalling that $f(\cdot)$ satisfies PLC we have
	\begin{equation}\label{eq:plc1}
	\|\nabla f(\bar{\x}_{t})\|^2 \geq  2\mu[f(\bar{\x}_{t})-f^\ast].
	\end{equation}
	Subtracting $f^\ast$ from both sides of \eqref{eq:sm1}, using \eqref{eq:plc1}, and noting $F(\bar{X}_{t}) = f(\bar{\x}_{t})$ yields
	\begin{equation}
	F(\bar{X}_{t+1})-f^\ast\leq [F(\bar{X}_{t})-f^\ast]\left(1-2\frac{\mu}{n}\eta(1-\frac{L\eta}{2})\right).
	\end{equation}
	Finally, recursively applying the above result establishes the stated expression.
\end{proof}
The next Lemma establishes a bound on the gradient norm  that will assist us in the proof of  Theorem 1.
\begin{lemma}\label{lem:plc2}
Recall $F$ is $L$-smooth and satisfies PLC with parameter $\mu$ on  $\mathcal{L}$. For any $\bar{X}\in \mathcal{L}$, let $\bar{X}^\ast$ be the projection of $\bar{X}$ on the solution set of $\min_\x f(\x)$. Then, we have
\begin{equation}
    \|\nabla F(\bar{X})-\nabla F(\bar{X}^\ast)\|^2 \leq \frac{L^2}{2\mu}(F(\bar{X})-f^\ast).
\end{equation}
\end{lemma}
\begin{proof}
First, by smoothness 
we have
\begin{equation}
    \|\nabla F(\bar{X})-\nabla F(\bar{X}^\ast)\|^2 \leq L^2 \|\bar{X} - \bar{X}^\ast\|^2.
\end{equation}
Additionally, PLC implies that $F$ satisfies the so-called quadratic growth condition (see Appendix A in \cite{karimi2016linear} as well as the related works \cite{bolte2017error,zhang2020new}),
\begin{equation}
   2\mu \|\bar{X} - \bar{X}^\ast\|^2 \leq F(\bar{X})-F^\ast, \text{ where } F^\ast = \min_X F(X). 
\end{equation}
Combining these two results and noting $F^\ast = f^\ast$ establishes the claim of the lemma.
\end{proof}
The next two Lemmas collectively establish a bound on amount of perturbation of the DeLi-CoCo's iterates $X^{(Q)}_t$ compared to 
$\bar{X}_{t}$.
{\color{black}
\begin{lemma}\label{lem:choco}
	Let $\bar{X}_{t+1} = \bar{X}_{t} - \eta\nabla F(\bar{X}_{t})$. Under the conditions of DeLi-CoCo, with $\gamma = \frac{\delta\sqrt{\omega}}{16\delta+\delta^2-8\delta\sqrt{\omega}+(4+2\delta)\lambda_{\max}^2(I-W)}$, it holds that
	\begin{equation}
	\E_{\mathcal{C}}\|X^{(Q)}_t - \bar{X}_{t}\|^2+\E_{\mathcal{C}}\|X^{(Q)}_t - Z^{(Q)}_t\|^2\leq (1-\frac{\delta\gamma}{2})^{Q}\left(\E_{\mathcal{C}}\|X^{(0)}_t - \bar{X}_{t}\|^2+\E_{\mathcal{C}}\|X^{(0)}_t - Z^{(0)}_t\|^2\right).
	\end{equation}
\end{lemma}
\begin{proof}
See Theorem 2 and its proof in \cite{koloskova2019decentralized}. The main ingredients of the proof are: (i) the fact that the error feedback sequence $Z^{(Q)}_t$ approaches $X^{(Q)}_t$ due to the contraction property of the compression operator $\mathcal{C}$, and (ii) the linear mixing rate of the gossiping matrix $W$ that is determined by the spectral gap $\delta$. 

As stated in the proof of Theorem 2 in \cite{koloskova2019decentralized} the choice of $\gamma = \frac{\delta\omega}{16\delta+\delta^2-8\delta\omega+(4+2\delta)\lambda_{\max}^2(I-W)}$ is suboptimal and results in a gossiping rate bounded above by $1-\frac{\delta^2\omega}{82}$. One drawback of such a sub-optimal gossip learning rate is that by relying on it we will not be able to show a theoretical improvement in the convergence properties of our proposed algorithm by performing $Q$ rounds of gossiping. Intuitively, consider two pairs of $(Q_1,\omega_1)$ and $(Q_1\times c,\omega_1/c)$ where $c>0$ is an integer. The proposed scheme for both of these pairs require the same amount of communication budget. Now using the results of  Theorem 2 in \cite{koloskova2019decentralized} and a simple Bernouli inequality we can observe that  
\begin{equation}
(1-\frac{\delta^2\omega_1}{82c})^{c\times Q_1} \approx (1-\frac{\delta^2\omega_1}{82})^{Q_1},
\end{equation}
i.e. the proposed scheme performs nearly identical in both instances. However, this is not consistent with the empirical observations that was demonstrated in Figure \ref{fig:motiv}. The reason for this inconsistency is the suboptimal choice of $\gamma$.
 Here, we perform a more refined approximation to circumvent this issue. Adopting the notations used in  the proof of Theorem 2 in \cite{koloskova2019decentralized} (see eq. 20 there) we propose to use
 \begin{equation}
 \alpha_1 = \frac{\gamma\delta}{2},\quad \alpha_2 = \frac{\sqrt{\omega}}{2}, \quad \gamma = \frac{\delta\sqrt{\omega}}{\underbrace{\{16\delta+\delta^2-8\delta\sqrt{\omega}+(4+2\delta)\lambda_{\max}^2(I-W)\}}_{<46}}.
 \end{equation}
 With these new options, the rest of the proof of Theorem 2 in \cite{koloskova2019decentralized} can be 
 easily adapted to show that Choco-gossip's rate of convergence (i.e. $1-\frac{\delta\gamma}{2}$) can be bounded by the tighter bound  $1-\frac{\delta^2\sqrt{\omega}}{82}$.
\end{proof}
}
\begin{lemma}\label{lem:plc3}
Let $\bar{X}_{t+1} = \bar{X}_t - \eta\nabla F(\bar{X}_t)$ and let $\{X^{(Q)}_t\}$ denote the sequence generated by DeLi-CoCo. Let $e_t^2 = \E_{\mathcal{C}}\|X^{(Q)}_t - \bar{X}_{t}\|^2+\E_{\mathcal{C}}\|X^{(Q)}_t - Z^{(Q)}_t\|^2$. For any $0<\eta\leq 1/L$, let $Q$ be such that
\begin{equation}\label{eq:simp-assp}
   \zeta :=  \rho-\xi>0, \quad \rho :=1-2\frac{\mu}{n}\eta(1-\frac{L\eta}{2})<1,\quad \xi:=(1-\frac{\delta\gamma}{2})^{Q} (3+20\eta^2 \hat{L}^2)<1.
\end{equation}
Then, it holds that
		\begin{equation}
	e_t^2 \leq \frac{13\eta^2L^2}{2\mu\zeta}(1-\frac{\delta\gamma}{2})^{Q}[F(\bar{X}_{0})-f^\ast]\rho^t+e_0^2\rho^t+\frac{20\eta^2\Delta^2}{1-\xi}(1-\frac{\delta\gamma}{2})^{Q}.
	\end{equation}
Additionally, if all nodes initialize such that $X_{0}^{Q} = 0$, by considering the dominant terms in above we have
\begin{equation}
	    	e_t^2 = \O\left(\frac{\eta^2L^2}{\mu\zeta}(1-\frac{\delta\gamma}{2})^{Q}[F(\bar{X}_{0})-f^\ast]\rho^t+\frac{\eta^2\Delta^2}{1-\xi}(1-\frac{\delta\gamma}{2})^{Q}\right),
	\end{equation}
	where the $\O$ notation does not hide any terms depending on $Q$ or $t$.
\end{lemma}
Before presenting the proof we highlight again that if $f$ is interpolating \cite{ma2017power,vaswani2018fast,vaswani2019painless}, e.g. an overparameterized neural network or a function satisfying the growth condition \cite{schmidt2013fast,cevher2019linear}, then $\Delta = 0$ and the second term disappears. Additionally, if there is no communication compression and the graph is fully connected ($\delta = 1$), the term $1-\gamma/2$ can be improved to $1-\gamma$ (see, e.g. \cite{xiao2004fast}). Therefore, by using the gossip learning rate $\gamma=1$, the second term in the error bound collapses to $0$. Further, \eqref{eq:simp-assp} is not necessary.
\begin{proof}
	It holds by definitions $X^{(0)}_{t+1}= X^{(Q)}_{t}-\eta\nabla F(X^{(Q)}_t)$ and $Z^{(0)}_{t+1} = Z^{(Q)}_{t}$ that
	\begin{equation}
	\begin{aligned}
	\E_{\mathcal{C}}&\|X^{(0)}_{t+1}-\bar{X}_{t+1}\|^2+\E_{\mathcal{C}}\|X^{(0)}_{t+1} - Z^{(0)}_{t+1}\|^2 \\
	&= \E_{\mathcal{C}}\|X^{(Q)}_t-\bar{X}_{t+1}-\bar{X}_{t}+\bar{X}_{t}-\eta\nabla F(X^{(Q)}_t) \|^2+\E_{\mathcal{C}}\|X^{(Q)}_t - Z^{(Q)}_t-\eta\nabla F(X^{(Q)}_t)\|^2\\
	&\leq \E_{\mathcal{C}}\left(3\|X^{(Q)}_t-\bar{X}_t\|^2+2\|X^{(Q)}_t - Z^{(Q)}_t\|^2+3\|\bar{X}_{t+1}-\bar{X}_{t}\|^2+5\eta^2 \|\nabla F(X^{(Q)}_t)\|^2\right)\\
&\leq  3e_t^2+3\eta^2 \|\mathcal{P}_{\mathcal{L}}(\nabla F(\bar{X}_{t}))\|^2+5\eta^2\E_{\mathcal{C}}\|\nabla F(X^{(Q)}_t)-\nabla F(\bar{X}_{t})+\nabla F(\bar{X}_{t})\|^2,
\end{aligned}
\end{equation}
where we used the fact that $\E_{\mathcal{C}}\|X^{(Q)}_t-Z^{(Q)}_t\|^2\geq 0$ and the smoothness of $F(\cdot)$. Let $\bar{X}^\ast_t$ be the projection of $\bar{X}_t$ to the optimal set. 
Now, we proceed by using the smoothness property, the non-expansiveness property of projection (cf. \eqref{eq:nexpproj}) as well as the fact that $\mathcal{P}_{\mathcal{L}}(\nabla F(X^\ast)) = \mathbf{0}$ to obtain
\begin{equation}\label{eq:errboundp1}
	\begin{aligned}
	\E_{\mathcal{C}}&\|X^{(0)}_{t+1}-\bar{X}_{t+1}\|^2+\E_{\mathcal{C}}\|X^{(0)}_{t+1} - Z^{(0)}_{t+1}\|^2 \\
&\leq  3e_t^2+3\eta^2 \|\mathcal{P}_{\mathcal{L}}(\nabla F(\bar{X}_{t}))-\mathcal{P}_{\mathcal{L}}(\nabla F(X^\ast))\|^2+5\eta^2\E_{\mathcal{C}}\|\nabla F(X^{(Q)}_t)-\nabla F(\bar{X}_{t})+\nabla F(\bar{X}_{t})\|^2\\
&\leq  3e_t^2+13\eta^2 \|\nabla F(\bar{X}_{t})-\nabla F(X^\ast)\|^2+10\eta^2\E_{\mathcal{C}}\|\nabla F(X^{(Q)}_t)-\nabla F(\bar{X}_{t})+\nabla F(X^\ast)\|^2\\
&\leq 3e_t^2+13\eta^2 \|\nabla F(\bar{X}_{t})-\nabla F(X^\ast)\|^2+20\hat{L}^2\eta^2\E_{\mathcal{C}}\|X^{(Q)}_t-\bar{X}_t\|^2+20\eta^2\|\nabla F(X^\ast)\|^2\\
&\leq 3e_t^2+13\eta^2 \|\nabla F(\bar{X}_{t})-\nabla F(X^\ast)\|^2+20\hat{L}^2\eta^2(\E_{\mathcal{C}}\|X^{(Q)}_t-\bar{X}_t\|^2+\E_{\mathcal{C}}\|X^{(Q)}_t-Z^{(Q)}_t\|^2)\\
&\qquad+20\eta^2\|\nabla F(X^\ast)\|^2\\
&\leq (3+20\eta^2 \hat{L}^2)e_t^2 +13\eta^2\|\nabla F(\bar{X}_t)-\nabla F(\bar{X}^\ast_t)\|^2+20\eta^2\Delta^2.
\end{aligned}
\end{equation}
To bound $\|\nabla F(\bar{X}_t)-\nabla F(\bar{X}^\ast_t)\|$ we can use the result of Lemma \ref{lem:plc2}. 
Thus, by Lemma \ref{lem:choco} and the above result we obtain the recursion
	\begin{equation}
	\begin{aligned}
	e_{t+1}^2 &\leq (1-\frac{\delta\gamma}{2})^{Q} \left((3+20\eta^2 \hat{L}^2)e_t^2+\frac{13}{2\mu}L^2\eta^2(F(\bar{X}_t)-f^\ast)+20\eta^2\Delta^2\right)\\
	& := \xi e_t^2 + \nu_t+u,
	\end{aligned}
	\end{equation}
	where 
	\begin{equation}
	\xi := (1-\frac{\delta\gamma}{2})^{Q} (3+20\eta^2 \hat{L}^2),
	\end{equation}
	\begin{equation}
	\nu_t := \frac{13}{2\mu}(1-\frac{\delta\gamma}{2})^{Q}L^2\eta^2(F(\bar{X}_t)-f^\ast)
	\end{equation}
	is a linearly decreasing sequence, i.e. $\nu_t \leq  \rho^t\nu_0$,
	where by Lemma \ref{lem:plc1} 
	\begin{equation}
	\rho := 1-2\frac{\mu}{n}\eta(1-\frac{L\eta}{2}), \qquad \nu_0 := \frac{13}{2\mu}(1-\frac{\delta\gamma}{2})^{Q}L^2\eta^2[F(\bar{X}_{0})-f^\ast],
	\end{equation}
	and 
	\begin{equation}
	    u:= 20(1-\frac{\delta\gamma}{2})^{Q}\eta^2\Delta^2.
	\end{equation}
Given the fact that $\nu_t $ vanishes linearly, we expect $e_t^2$ to converge linearly because for a large enough $Q$, we have $\xi < 1$ (see the conditions in the statement of Theorem 1).
We now prove this statement using induction.

Define $h_t$ such that $h_0 = e_0^2$, and let $h_{t+1}= a h_t+\nu_t$. Using simple algebra it follows that 
\begin{equation}
    h_t = \xi^t e_0^2+\sum_{i=0}^{t-1}\xi^{t-i-1}\rho^i\nu_0.
\end{equation}
Similarly, we can expand the recursion of $e^2_t$ to obtain
\begin{equation}
    e^2_t \leq \xi^t e_0^2+\sum_{i=0}^{t-1}\xi^{t-i-1}\rho^i\nu_0 + u \sum_{i=0}^{t-1}\xi^i \leq h_t + \frac{u}{1-\xi}.
\end{equation}
Note that if $h_t$ linearly converges to zero, then $e^2_t$ linearly converges as well. 
Since $\rho-\xi=\zeta>0$ by assumption, using simple algebra we can show
\begin{equation}
    h_t\leq \rho^t(e^2_0+\frac{\nu_0}{\zeta}).
\end{equation}
Thus, 
\begin{equation}
    e^2_t \leq \rho^t(e_0^2+\frac{\nu_0}{\zeta})+\frac{u}{1-\xi},
\end{equation}
and the proof is complete by noting the definitions of $\rho$, $\xi$, $\nu_0$,  $u$, and the fact that by definition $e^2_0 = \|X_0^{(Q)}\|^2$.
\end{proof}
\subsubsection*{A.1.1 Proof of Theorem 1}
As stated in the main paper, the main challenge in decentralized learning under PLC is that we cannot use the  co-coercivity property.
To this end, we leverage the fact that PLC relates to suboptimality of the function value and exploit a judiciously chosen step size to simplify the analysis. 

Let $X^{(q)}_t = \bar{X}_t+E^{(q)}_t$ for some (random) error matrix $E^{(q)}_t \in \mathbb{R}^{d\times n}$. By $\hat{L}$-smoothness of $F(\cdot)$ for all $X $ and the fact that $\bar{X}_{t+1} = \bar{X}_t-\eta \mathcal{P}_{\mathcal{L}}(\nabla F(\bar{X}_t))$ we have
\begin{equation}
F(X^{(Q)}_{t+1})\leq F(\bar{X}_t)+\langle-\eta\mathcal{P}_{\mathcal{L}}(\nabla F(\bar{X}_t))+ E^{(Q)}_{t+1},\nabla F(\bar{X}_t)\rangle +\frac{\hat{L}}{2}\|\eta\nabla F(\bar{X}_t)- E^{(Q)}_{t+1}\|^2.
\end{equation}
Let $\eta = 1/\hat{L}\leq 1/L$ and hence the condition of Lemma \ref{lem:plc1} is satisfied. Expanding the last inequality  by using $\eta = 1/\hat{L}$ we obtain
\begin{equation}\label{firststepp}
\begin{aligned}
F(X^{(Q)}_{t+1})&\leq F(\bar{X}_t)-\frac{1}{\hat{L}}\langle\mathcal{P}_{\mathcal{L}}(\nabla F(\bar{X}_t)),\nabla F(\bar{X}_t)\rangle+\langle\nabla F(\bar{X}_t)-\mathcal{P}_{\mathcal{L}}(\nabla F(\bar{X}_t)),E^{(Q)}_{t+1}\rangle\\
&\qquad+\frac{\hat{L}}{2}\|E^{(Q)}_{t+1}\|^2+\frac{1}{2\hat{L}}\|\mathcal{P}_{\mathcal{L}}(\nabla F(\bar{X}_t))\|^2.
\end{aligned}
\end{equation}
Next, we need to take care of the cross terms. First, note that 
\begin{equation}\label{eq:negres}
\begin{aligned}
-\langle\mathcal{P}_{\mathcal{L}}(\nabla F(\bar{X}_t)),\nabla F(\bar{X}_t)\rangle &= \langle-\mathcal{P}_{\mathcal{L}}(\nabla F(\bar{X}_t)),\nabla F(\bar{X}_t)+\mathcal{P}_{\mathcal{L}}(\nabla F(\bar{X}_t))-\mathcal{P}_{\mathcal{L}}(\nabla F(\bar{X}_t))\rangle\\
& = -\|\mathcal{P}_{\mathcal{L}}(\nabla F(\bar{X}_t))\|^2+ \langle\mathbf{0}-\mathcal{P}_{\mathcal{L}}(\nabla F(\bar{X}_t)),\nabla F(\bar{X}_t)-\mathcal{P}_{\mathcal{L}}(\nabla F(\bar{X}_t))\rangle\\
&\leq -\|\mathcal{P}_{\mathcal{L}}(\nabla F(\bar{X}_t))\|^2,
\end{aligned}
\end{equation}
where the last inequality follows by the variational characterization of the projection (see Lemma 2.2.7 in \cite{nesterov2013introductory}), i.e.
\begin{equation}
    \langle Y-\mathcal{P}_{\mathcal{L}}(Y),X - \mathcal{P}_{\mathcal{L}}(Y)\rangle\leq 0,\quad X \in \mathcal{L},
\end{equation}
and the fact that $\mathbf{0}\in \mathcal{L}$. 

We now bound the second cross term in \eqref{firststepp}. From Young's inequality, for all $\alpha>0$ it holds that
\begin{equation}
\begin{aligned}
\langle\nabla F(\bar{X}_t)&-\mathcal{P}_{\mathcal{L}}(\nabla F(\bar{X}_t)),E^{(Q)}_{t+1}\rangle \\
&\leq \frac{\alpha}{2}\|E^{(Q)}_{t+1}\|^2+\frac{1}{2\alpha}\|\nabla F(\bar{X}_t)-\mathcal{P}_{\mathcal{L}}(\nabla F(\bar{X}_t))\|^2\\
&=\frac{\alpha}{2}\|E^{(Q)}_{t+1}\|^2+\frac{1}{2\alpha}\|\nabla F(\bar{X}_t)\|^2+\frac{1}{2\alpha}\|\mathcal{P}_{\mathcal{L}}(\nabla F(\bar{X}_t))\|^2-\frac{1}{\alpha}\langle\mathcal{P}_{\mathcal{L}}(\nabla F(\bar{X}_t)),\nabla F(\bar{X}_t)\rangle\\
&\leq \frac{\alpha}{2}\|E^{(Q)}_{t+1}\|^2+\frac{1}{2\alpha}\|\nabla F(\bar{X}_t)\|^2-\frac{1}{2\alpha}\|\mathcal{P}_{\mathcal{L}}(\nabla F(\bar{X}_t))\|^2 \text{ \quad (Using \eqref{eq:negres})}\\
& \leq \frac{\alpha}{2}\|E^{(Q)}_{t+1}\|^2 + \frac{1}{2\alpha}\|\nabla F(\bar{X}_t) - \nabla F(\bar{X}^\ast_t) + \nabla F(\bar{X}^\ast_t)\|^2\\
& \text{ (where $\bar{X}^\ast_t$ is the projection of $\bar{X}_t$ on solution set of $\min_\x f(\x) = \mathcal{X}^\ast$ -- see Lemma \ref{lem:plc2})}\\
&\leq\frac{\alpha}{2}\|E^{(Q)}_{t+1}\|^2+\frac{1}{\alpha}\|\nabla F(\bar{X}_t)-\nabla F(\bar{X}^\ast_t)\|^2+\frac{1}{\alpha}\Delta^2\\
& \text{ (recall } \Delta^2 :=\max_{\x^\ast \in \mathcal{X}^\ast}\sum_{i=1}^n \|\nabla f_i(\x^\ast)\|^2) \\
&\leq\frac{\alpha}{2}\|E^{(Q)}_{t+1}\|^2+\frac{L^2}{2\alpha\mu}[F(\bar{X}_t)-f^\ast]+\frac{1}{\alpha}\Delta^2,
\end{aligned}
\end{equation}
In the last step, we use the result of Lemma \ref{lem:plc2} by noting the fact that $F$ on $\mathcal{L}$, similar to $f$, satisfies the PL condition with parameter $\mu$.

Putting the bounds on the cross terms in \eqref{firststepp} together, we obtain
\begin{equation}\label{eq:firststep}
\begin{aligned}
F(X^{(Q)}_{t+1})&\leq F(\bar{X}_t)-\frac{1}{2\hat{L}}\|\mathcal{P}_{\mathcal{L}}(\nabla F(\bar{X}_t))\|^2+\frac{\hat{L}+\alpha}{2}\|E^{(Q)}_{t+1}\|^2+\frac{L^2}{2\alpha\mu}[F(\bar{X}_t)-f^\ast]+\frac{1}{\alpha}\Delta^2.
\end{aligned}
\end{equation}
Subtracting $f^\ast$ from both sides, using the PL condition of $f$ along the fact that $\|\mathcal{P}_{\mathcal{L}}(\nabla F(\bar{X}_t))\|^2 = \|\nabla f(\bar{\x}_t)\|^2/n$ and $f(\bar{\x}_t) = F(\bar{X}_t)$, and taking the expectation yields 
\begin{equation}\label{eq:firststep1}
\begin{aligned}
\E_{\mathcal{C}}[F(X^{(Q)}_{t+1})]-f^\ast \leq  (1-\frac{\mu}{n\hat{L}}+\frac{L^2}{2\alpha\mu}) &\left(F(\bar{X}_t)-f^\ast\right)+ \frac{\hat{L}+\alpha}{2}\E_{\mathcal{C}}\|E^{(Q)}_{t+1}\|^2+\frac{1}{\alpha}\Delta^2. 
\end{aligned}
\end{equation}
Recall from Lemma \ref{lem:plc1} and \ref{lem:plc3} that with the specific choice of $\eta = 1/\hat{L}$ and the fact that $\hat{L}\geq L$ we have
\begin{equation}
F(\bar{X}_t) - f^\ast \leq [F(\bar{X}_0)-f^\ast]\rho^t
\end{equation}
and 
\begin{equation}
\E_{\mathcal{C}}\|E^{(Q)}_{t}\|^2 \leq e_t^2 \leq \frac{13\eta^2L^2}{2\mu\zeta}(1-\frac{\delta\gamma}{2})^{Q}[F(\bar{X}_{0})-f^\ast]\rho^t+\|X^{(Q)}_0\|^2\rho^t+\frac{20\eta^2\Delta^2}{1-\xi}(1-\frac{\delta\gamma}{2})^{Q},
	\end{equation}
where
\begin{equation}\label{eq:somedef}
   \zeta :=  \rho-\xi>0, \quad \rho :=1-2\frac{\mu}{n\hat{L}}(1-\frac{L}{2\hat{L}})<1-\frac{\mu}{n\hat{L}},\quad \xi:=23(1-\frac{\delta\gamma}{2})^{Q}<\rho.
\end{equation}
To state a simple and clear result, we make a simplifying assumption that $\alpha \geq \hat{L}$. Indeed, in our regime of interest where $(1-\frac{\delta\gamma}{2})^Q$ is small, if
\begin{equation}
    \alpha = \frac{\hat{L}\sqrt{1-\xi}}{(1-\frac{\delta\gamma}{2})^{\frac{Q}{2}}},
\end{equation}
following simple algebra we can show this assumption holds if $Q$ satisfies
\begin{equation}
    24(1-\frac{\delta\gamma}{2})^{Q}<1.
\end{equation}
Therefore, 
\begin{equation}
    \frac{\hat{L}+\alpha}{2} \leq \alpha,\quad (1-\frac{\mu}{n\hat{L}}+\frac{L^2}{2\alpha\mu}) \leq 1-\frac{\mu}{n\hat{L}}+\frac{L}{2\mu}.
\end{equation}
Thus, it holds that 
\begin{equation}\label{eq:last1}
\begin{aligned}
&\E_{\mathcal{C}}[F(X^{(Q)}_{t})]-f^\ast \leq \frac{21\Delta^2}{\hat{L}(1-\xi)}(1-\frac{\delta\gamma}{2})^{\frac{Q}{2}}\\
&+ \rho^t \left(\left[\frac{1-\frac{\mu}{n\hat{L}}+\frac{L}{2\mu}}{\rho}+\frac{13L\sqrt{1-\xi}}{2\mu\zeta}\left(1-\frac{\delta^2\sqrt{\omega}}{82}\right)^{\frac{Q}{2}}\right]\left[F(\bar{X}_{0})-f^\ast\right]+\frac{\|X^{(Q)}_0\|^2\hat{L}\sqrt{1-\xi}}{\left(1-\frac{\delta^2\sqrt{\omega}}{82}\right)^{\frac{Q}{2}}}\right).
\end{aligned}
\end{equation}
That is, the above result establishes the linear convergence of the proposed scheme under smoothness and PLC.

Finally, to obtain the stated result, we first assume that each node is initialized at zero, i.e.  $\|X^{(Q)}_0\|^2 = 0$ which in turn simplifies the last term in \eqref{eq:last1} to zero. Additionally, given the definition $\rho :=1-\frac{2\mu}{n\hat{L}}+\frac{\mu L}{n\hat{L}^2}< 1-\frac{\mu}{n\hat{L}}$,  $\frac{1-\frac{\mu}{n\hat{L}}+\frac{L}{2\mu}}{\rho}$ can be approximated by $1+\frac{L}{2\mu\rho}$. Furthermore, since $\xi<\rho<1$ (see \eqref{eq:somedef}), assuming
\begin{equation}
23(1-\frac{\delta\gamma}{2})^{Q}<\frac{\rho}{2},
\end{equation}
we can write
\begin{equation}
\frac{\sqrt{1-\xi}}{\zeta}\leq \frac{2}{\rho},
\end{equation}
thereby establishing the stated bound.
\subsection*{A.2. Proof of Convergence Under Strong Convexity}
The proof follows a near identical perturbation analysis in Theorem 1.

To formalize the arguments, we start by providing some intermediate lemmas. The first Lemma establishes linear convergence of the ``unperturbed sequence''. 
\begin{lemma}\label{lem:sc1}
	Let $\bar{X}_{t+1} = \bar{X}_{t} - \eta\mathcal{P}_{\mathcal{L}}(\nabla F(\bar{X}_{t}))$ denote a sequence such that $\bar{X}_{0} \in \mathcal{L}$, $\bar{X}_{t} = [\bar{\x}_{t},\dots,\bar{\x}_{t}]$, and $\eta\leq 2/(L+\mu)$. Then
	\begin{equation}
	\|\bar{X}_{t} -X^\ast\| \leq (1-\mu\eta)^t\|\bar{X}_{0} -X^\ast\|.
	\end{equation}
\end{lemma}
\begin{proof}
Given the update of the average iterates it holds that
	 \begin{equation}\label{eq:barupdate1}
	 \begin{aligned}
	     \|\bar{X}_{t+1} -X^\ast\|^2 &= \|\mathcal{P}_{\mathcal{L}}(\bar{X}_t-\eta \nabla F(\bar{X}_t))-\mathcal{P}_{\mathcal{L}}(X^\ast-\nabla F(X^\ast))\|^2\\
	     &\leq \|\bar{X}_t-\eta \nabla F(\bar{X}_t)-X^\ast+\eta\nabla F(X^\ast)\|^2\\
	     =& \|\bar{X}_{t} -X^\ast\|^2+\eta^2\|\nabla F(\bar{X}_t)-\nabla F(X^\ast)\|^2-2\eta \langle\nabla F(\bar{X}_t)-\nabla F(X^\ast),\bar{X}_{t} -X^\ast\rangle,
	 \end{aligned}
	 \end{equation}
	where to obtain the inequality we use the non-expansiveness of projection (see Lemma 2.2.7 and  Corollary 2.2.3 in \cite{nesterov2013introductory}),
\begin{equation}
    \|\mathcal{P}_{\mathcal{L}}(X)-\mathcal{P}_{\mathcal{L}}(Y)\|\leq \|X-Y\|.
\end{equation}	
Now, we use Theorem 2.1.11 in \cite{nesterov2013introductory}, i.e.,
\begin{equation}
\langle\nabla F(Z)-\nabla F(Y) , Z-Y\rangle \geq \frac{{\mu}{L}}{{\mu}+{L}}\|Z-Y\|^2+\frac{1}{{\mu}+{L}}\|F(Z)-\nabla F(Y)\|^2, \quad Z,Y \in \mathcal{L}
\end{equation}
for $Z = \bar{X}_t$ and $Y = X^\ast$ to bound the inner-product on the RHS of \eqref{eq:barupdate1},
\begin{equation}\label{eq:barupdate2}
\begin{aligned}
\|\bar{X}_{t+1} -X^\ast\|^2 &\leq (1-2\eta\frac{\mu L}{\mu+L})\|\bar{X}_{t} -X^\ast\|^2+ (\eta^2-\frac{2\eta}{\mu+L})\|\nabla F(\bar{X}_t)-\nabla F(X^\ast)\|^2\\
&\leq (1-2\eta\frac{\mu L}{\mu+L}+\eta^2\mu^2-2\eta\frac{\mu^2}{\mu+L})\|\bar{X}_{t} -X^\ast\|^2\\
&= (1-\eta\mu)^2 \|\bar{X}_{t} -X^\ast\|^2,
\end{aligned}
\end{equation}
where to obtain the inequality we use the fact that $\eta \leq 2/(L+\mu)$ as well as the strong convexity of $F(\cdot)$ -- in particular, the inequality
	 \begin{equation}
\| \nabla F(\bar{X}_{t})-\nabla F(X^\ast)\| \geq \mu \|\bar{X}_{t}-X^\ast\|.
\end{equation}
	Finally, recursively applying the result of \eqref{eq:barupdate2} establishes the stated expression.
\end{proof}
The next  Lemma establishes a bound on the amount of perturbation of the DeLi-CoCo's iterates $X^{(Q)}_t$ compared to 
$\bar{X}_{t}$.
\begin{lemma}\label{lem:sc2}
Let $\bar{X}_{t+1} = \bar{X}_t - \eta\mathcal{P}_{\mathcal{L}}(\nabla F(\bar{X}_t))$ and let $\{X^{(Q)}_t\}$ denote the sequence generated by DeLi-CoCo. Let $e_t^2 = \E_{\mathcal{C}}\|X^{(Q)}_t - \bar{X}_{t}\|^2+\E_{\mathcal{C}}\|X^{(Q)}_t - Z^{(Q)}_t\|^2$. For any $0<\eta\leq 2/(L+\mu)$, let $Q$ be such that 
	\begin{equation}\label{eq:rho}
	    \zeta:= \rho-\xi>0,\quad \rho:=(1-\eta\mu)^2<1,\quad \xi:=(1-\frac{\delta\gamma}{2})^{Q}(3+20\eta^2 \hat{L}^2)<1.
	\end{equation}
	Then, it holds that 
	\begin{equation}
	e_t^2 \leq \frac{13 L\eta^2}{\zeta}(1-\frac{\delta\gamma}{2})^{Q}\|X_0^{(Q)}-X^\ast\|^2 \rho^t+\rho^t \|X_0^{(Q)}\|^2+\frac{20\eta^2\Delta^2}{1-\xi}(1-\frac{\delta\gamma}{2})^{Q}.
	\end{equation}
	Additionally, if all nodes initialize such that $X_{0}^{Q} = 0$, by considering the dominant terms in above we obtain
\begin{equation}
	e_t^2 = \O\left(\frac{L\eta^2}{\zeta}(1-\frac{\delta\gamma}{2})^{Q}\|\bar{X}_0-X^\ast\|^2 \rho^t+\frac{\eta^2\Delta^2}{1-\xi}(1-\frac{\delta\gamma}{2})^{Q} \right),
	\end{equation}
	where the $\O$ notation does not hide any term depending on $Q$ or $t$.
\end{lemma}
\begin{proof}
It holds by definitions $X^{(0)}_{t+1}= X^{(Q)}_{t}-\eta\nabla F(X^{(Q)}_t)$ and $Z^{(0)}_{t+1} = Z^{(Q)}_{t}$ that
\begin{equation}
\begin{aligned}
\E_{\mathcal{C}}&\|X^{(0)}_{t+1}-\bar{X}_{t+1}\|^2+\E_{\mathcal{C}}\|X^{(0)}_{t+1} - Z^{(0)}_{t+1}\|^2 \\
&= \E_{\mathcal{C}}\|X^{(Q)}_t-\bar{X}_{t+1}-\bar{X}_{t}+\bar{X}_{t}-\eta\nabla F(X^{(Q)}_t) \|^2+\E_{\mathcal{C}}\|X^{(Q)}_t - Z^{(Q)}_t-\eta\nabla F(X^{(Q)}_t)\|^2\\
&\leq \E_{\mathcal{C}}\left(3\|X^{(Q)}_t-\bar{X}_t\|^2+2\|X^{(Q)}_t - Z^{(Q)}_t\|^2+3\|\bar{X}_{t+1}-\bar{X}_{t}\|^2+5\eta^2 \|\nabla F(X^{(Q)}_t)\|^2\right)\\
&\leq  3e_t^2+3\eta^2 \|\mathcal{P}_{\mathcal{L}}(\nabla F(\bar{X}_{t}))\|^2+5\eta^2\E_{\mathcal{C}}\|\nabla F(X^{(Q)}_t)-\nabla F(\bar{X}_{t})+\nabla F(\bar{X}_{t})\|^2,
\end{aligned}
\end{equation}
where we used the update rule of the average iterates $\bar{X}_{t+1}= \bar{X}_{t}-\eta\mathcal{P}_{\mathcal{L}}(\nabla F(\bar{X}_{t}))$. Now, we proceed by using the smoothness property of $F(\cdot)$, the non-expansiveness property of projection (cf. \eqref{eq:nexpproj}) as well as the fact that $\mathcal{P}_{\mathcal{L}}(\nabla F(X^\ast)) = \mathbf{0}$ to obtain
\begin{equation}\label{eq:errbound1}
\begin{aligned}
\E_{\mathcal{C}}&\|X^{(0)}_{t+1}-\bar{X}_{t+1}\|^2+\E_{\mathcal{C}}\|X^{(0)}_{t+1} - Z^{(0)}_{t+1}\|^2 \\
&\leq  3e_t^2+3\eta^2 \|\mathcal{P}_{\mathcal{L}}(\nabla F(\bar{X}_{t}))-\mathcal{P}_{\mathcal{L}}(\nabla F(X^\ast))\|^2+5\eta^2\E_{\mathcal{C}}\|\nabla F(X^{(Q)}_t)-\nabla F(\bar{X}_{t})+\nabla F(\bar{X}_{t})\|^2\\
&\leq  3e_t^2+13\eta^2 \|\nabla F(\bar{X}_{t})-\nabla F(X^\ast)\|^2+10\eta^2\E_{\mathcal{C}}\|\nabla F(X^{(Q)}_t)-\nabla F(\bar{X}_{t})+\nabla F(X^\ast)\|^2\\
&\leq 3e_t^2+13\eta^2 \|\nabla F(\bar{X}_{t})-\nabla F(X^\ast)\|^2+20\hat{L}^2\eta^2\E_{\mathcal{C}}\|X^{(Q)}_t-\bar{X}_t\|^2+20\eta^2\|\nabla F(X^\ast)\|^2\\
&\leq 3e_t^2+13\eta^2 \|\nabla F(\bar{X}_{t})-\nabla F(X^\ast)\|^2+20\hat{L}^2\eta^2(\E_{\mathcal{C}}\|X^{(Q)}_t-\bar{X}_t\|^2+\E_{\mathcal{C}}\|X^{(Q)}_t-Z^{(Q)}_t\|^2)\\
&\qquad+20\eta^2\|\nabla F(X^\ast)\|^2\\
&\leq (3+20\eta^2 \hat{L}^2)e_t^2 +13\eta^2L^2 \|\bar{X}_t-X^\ast\|^2+20\eta^2\|\nabla F(X^\ast)\|^2,
\end{aligned}
\end{equation}
where we used the fact that $\E_{\mathcal{C}}\|X^{(Q)}_t-Z^{(Q)}_t\|^2\geq 0$ and exploited smoothness of $F(\cdot)$.
Thus, by Lemma \ref{lem:choco} and the result of \eqref{eq:errbound1} we obtain the recursion
\begin{equation}
\begin{aligned}
e_{t+1}^2 \leq \xi e_t^2 + \nu_t + u,
\end{aligned}
\end{equation}
where 
\begin{equation}
\xi := (1-\frac{\delta\gamma}{2})^{Q} (3+20\eta^2 \hat{L}^2),
\end{equation}
\begin{equation}
\nu_t := 13L\eta^2(1-\frac{\delta\gamma}{2})^{Q}\|\bar{X}_t-X^\ast\|^2,
\end{equation}
and
\begin{equation}
    u:=20(1-\frac{\delta\gamma}{2})^{Q} \eta^2\|\nabla F(X^\ast)\|^2 =20(1-\frac{\delta\gamma}{2})^{Q} \eta^2\Delta^2.
\end{equation}
Note that by Lemma \ref{lem:sc1}, $\nu_t $ is a linearly converging sequence, i.e. $\nu_t \leq  \rho^t\nu_0$
where 
\begin{equation}
\rho := (1-\eta\mu)^2, \qquad \nu_0 := 13L\eta^2(1-\frac{\delta\gamma}{2})^{Q}\|\bar{X}_0-X^\ast\|^2.
\end{equation}
Since for a large enough $Q$ and small enough $\eta$ (see the conditions in the statement of Theorem \ref{thm:1}) $a<1$, given the fact that $\nu_t $ vanishes linearly, we expect  $e_t^2$ to converge linearly. In the following we prove this statement using induction.

Define $h_t$ such that $h_0 = e_0^2$, and $h_{t+1}= \xi h_t+\nu_t$. Using simple algebra it follows that 
\begin{equation}
    h_t = \xi^t e_0^2+\sum_{i=0}^{t-1}\xi^{t-i-1}\rho^i\nu_0,
\end{equation}
Similarly, we can expand the recursion of $e^2_t$ to obtain
\begin{equation}
    e^2_t \leq \xi^t e_0^2+\sum_{i=0}^{t-1}\xi^{t-i-1}\rho^i\nu_0 + u \sum_{i=0}^{t-1}\xi^i \leq h_t + \frac{u}{1-\xi}.
\end{equation}
Note that if $h_t$ linearly converges to zero, hence, $e^2_t$ linearly converges as well. 
Since $\rho-\xi=\zeta>0$ by assumption, using simple algebra we can show
\begin{equation}
    h_t\leq \rho^t(e^2_0+\frac{\nu_0}{\zeta}).
\end{equation}
Thus, 
\begin{equation}
    e^2_t \leq \rho^t(e_0^2+\frac{\nu_0}{\zeta})+\frac{u}{1-\xi},
\end{equation}
and the proof is complete by noting the definitions of $\xi$, $\rho$, $\nu_0$, and $u$.
\end{proof}
\subsubsection*{A.2.1 Proof of Theorem \ref{thm:1}}
Let $X^{(q)}_t = \bar{X}_t+E^{(q)}_t $ for some (random) error matrix $E^{(q)}_t  \in \mathbb{R}^{d\times n}$. Recall $\mathcal{P}_{\mathcal{L}}(.)$ denotes the projection operator onto the linear subspace (denoted by $\mathcal{L}$) of $d\times n$ matrices with identical columns. Noting the fact that $X^\ast = \mathcal{P}_{\mathcal{L}}(\bar{X}^\ast-\eta \nabla F(\bar{X}^\ast))$ and $\bar{X}_{t+1} = \mathcal{P}_{\mathcal{L}}(\bar{X}_{t} - \eta\nabla F(\bar{X}_{t}))$, and using the update rule of DeLi-CoCo we have
\begin{equation}\label{eq:thm1:m}
\begin{aligned}
\|X_{t+1}^{(Q)}-X^\ast\|^2 &=\|E^{(Q)}_{t+1}\|^2+\|\mathcal{P}_{\mathcal{L}}(X^{(Q)}_t-\eta \nabla F(X^{(Q)}_t))-\mathcal{P}_{\mathcal{L}}(X^\ast-\eta \nabla F(X^\ast))\|^2\\&\qquad+2\langle E^{(Q)}_{t+1},\mathcal{P}_{\mathcal{L}}(X^{(Q)}_t-\eta \nabla F(X^{(Q)}_t))-\mathcal{P}_{\mathcal{L}}(X^\ast) \rangle \\
&\leq\|E^{(Q)}_{t+1}\|^2+2 \langle E^{(Q)}_{t+1},\mathcal{P}_{\mathcal{L}}(X^{(Q)}_t-\eta \nabla F(X^{(Q)}_t))-\mathcal{P}_{\mathcal{L}}(X^\ast-\eta \nabla F(X^\ast)) \rangle\\&\qquad+\|X^{(Q)}_t-\bar{X}^\ast-\eta\nabla F(X^{(Q)}_t)+\eta \nabla F(X^\ast)\|^2,
\end{aligned}
\end{equation}
where to obtain the inequality we used the non-expansiveness property of projection (cf. \eqref{eq:nexpproj}). Next, we aim to bound each of the individual terms above. First, we know from the definition of $e_{t}$ that $\E_{\mathcal{C}}\|E^{(Q)}_{t+1}\| \leq e_{t+1}$. Secondly, by using the Cauchy-Schwarz inequality, the cross term can be entangled and dealt with according to 
\begin{equation}
    \begin{aligned}
    \E_{\mathcal{C}}[\langle E^{(Q)}_{t+1},&\mathcal{P}_{\mathcal{L}}(X^{(Q)}_t-\eta \nabla F(X^{(Q)}_t))-\mathcal{P}_{\mathcal{L}}(X^\ast-\eta \nabla F(X^\ast)) \rangle] \\&\leq \E_{\mathcal{C}}\|E^{(Q)}_{t+1}\|\|\mathcal{P}_{\mathcal{L}}(X^{(Q)}_t-\eta \nabla F(X^{(Q)}_t))-\mathcal{P}_{\mathcal{L}}(X^\ast-\eta \nabla F(X^\ast))\|\\&\leq\E_{\mathcal{C}}\|E^{(Q)}_{t+1}\|\|X^{(Q)}_t-\bar{X}^\ast-\eta\nabla F(X^{(Q)}_t)+\eta \nabla F(X^\ast)\|,
    \end{aligned}
\end{equation}
by using the non-expansiveness property of projection (cf. \eqref{eq:nexpproj}). Thus, 
by taking the expectation of both sides in \eqref{eq:thm1:m} and using the above arguments we obtain
\begin{equation}
\begin{aligned}
\E_{\mathcal{C}}\|X_{t+1}^{(Q)}-X^\ast\|^2&\leq e_{t+1}^2+2 e_{t+1} \E_{\mathcal{C}}\|X^{(Q)}_t-X^\ast-\eta\nabla F(X^{(Q)}_t)+\eta \nabla F(X^\ast)\|\\&\qquad+\E_{\mathcal{C}}\|X^{(Q)}_t-X^\ast-\eta\nabla F(X^{(Q)}_t)+\eta \nabla F(X^\ast)\|^2.
\end{aligned}
\end{equation}
Now we bound the last term on the RHS of the above expression. Using Theorem 2.1.11 in \cite{nesterov2013introductory}, i.e.
\begin{equation}
\langle\nabla F(Z)-\nabla F(Y) , Z-Y\rangle \geq \frac{\hat{\mu}\hat{L}}{\hat{\mu}+\hat{L}}\|Z-Y\|^2+\frac{1}{\hat{\mu}+\hat{L}}\|F(Z)-\nabla F(Y)\|^2,
\end{equation}
for $Z = X^{(Q)}_t$ and $Y = X^\ast$ we have 
\begin{equation}
\begin{aligned}
\|X^{(Q)}_t-X^\ast-&\eta\nabla F(X^{(Q)}_t)+\eta \nabla F(X^\ast)\|^2\leq \|X^{(Q)}_t-X^\ast\|^2+\eta^2\|\nabla F(X^{(Q)}_t)-\nabla F(X^\ast)\|^2\\&-2\frac{\hat{\mu}\hat{L}}{\hat{\mu}+\hat{L}}\eta \|X^{(Q)}_t-X^\ast\|^2-\frac{2\eta}{\hat{\mu}+\hat{L}}\|\nabla F(X^{(Q)}_t)- \nabla F(X^\ast)\|^2\\
&= \|X^{(Q)}_t-X^\ast\|^2 \left(1-2\frac{\hat{\mu}\hat{L}}{\hat{\mu}+\hat{L}}\eta\right)+\|\nabla F(X^{(Q)}_t)- \nabla F(X^\ast)\|^2\left(\eta-\frac{2}{\hat{\mu}+\hat{L}}\right)\eta\\
&\leq \|X^{(Q)}_t-X^\ast\|^2 \left(1-2\frac{\hat{\mu}\hat{L}}{\hat{\mu}+\hat{L}}\eta+\eta(\eta-\frac{2}{\hat{\mu}+\hat{L}})\hat{\mu}^2\right)\\ &= \|X^{(Q)}_t-X^\ast\|^2 (1-\eta\hat{\mu})^2:= \|X^{(Q)}_t-X^\ast\|^2 \ell^2 ,
\end{aligned}
\end{equation}
for any $\eta \leq 2/(\hat{L}+\hat{\mu})$. The strong convexity of $F(\cdot)$ implies
\begin{equation}
\| \nabla F(X^{(Q)}_t)-\nabla F(X^\ast)\| \geq \hat{\mu} \|X^{(Q)}_t-X^\ast\|.
\end{equation}
Thus, we can now put together a bound on the error term according to
\begin{equation}
\begin{aligned}
\E_{\mathcal{C}}\|X^{(Q)}_{t+1}-X^\ast\|^2  &\leq \ell^2 \E_{\mathcal{C}}\|X^{(Q)}_t-X^\ast\|^2+e_{t+1}^2 + 2\ell   e_{t+1}\E_{\mathcal{C}}\|X^{(Q)}_t-X^\ast\| \\
&= (e_{t+1}+\ell  \E_{\mathcal{C}}\|X^{(Q)}_t-X^\ast\|)^2.
\end{aligned}
\end{equation}
Therefore,
\begin{equation}
\E_{\mathcal{C}}\|X^{(Q)}_{t+1}-X^\ast\| \leq   e_{t+1}+\ell \E_{\mathcal{C}} \|X^{(Q)}_t-X^\ast\|.
\end{equation}
Recall from Lemma \ref{lem:sc2} that the sequence $e_{t}^2$ converges linearly, i.e.,
	\begin{equation}
	e_t^2 \leq \frac{13 L\eta^2}{\zeta}(1-\frac{\delta\gamma}{2})^{Q}\|X_0^{(Q)}-X^\ast\|^2 \rho^t+\rho^t \|X_0^{(Q)}\|^2+\frac{20\eta^2\Delta^2}{1-\xi}(1-\frac{\delta\gamma}{2})^{Q},
	\end{equation}
where
	\begin{equation}
	    \zeta:= \rho-\xi>0,\quad \rho:=(1-\eta\mu)^2<1,\quad \xi:=(1-\frac{\delta\gamma}{2})^{Q}(3+20\eta^2 \hat{L}^2)<1.
	\end{equation}
Note that for any $\eta$ we have $\ell\geq \rho$ and we can upperbound the bound on $e_t^2$ by replacing $\rho$ with $\ell$. Using a similar technique as the one we used towards the end of the proof of Lemma \ref{lem:sc2}, we can show for any two sequence $h_t^{(1)}$ and $h_t^{(2)} := \ell^t h_0^{(2)}$ that satisfy
\begin{equation}
    h_t^{(1)} \leq \ell h_t^{(1)} +h_t^{(2)}+u^{(1)},
\end{equation}
it holds that 
\begin{equation}
    h_t^{(1)} \leq \ell^t (h_0^{(1)}+t\frac{h_0^{(2)}}{\ell}) +\frac{u^{(1)}}{1-\ell}.
\end{equation}
Thus, replacing $h_t^{(1)}$ and $h_t^{(2)}$ with $\E_{\mathcal{C}}\|X^{(Q)}_{t}-X^\ast\|^2$ and $e^2_t$ we obtain that $\E_{\mathcal{C}}\|X^{(Q)}_{t+1}-X^\ast\|$ converges according to
\begin{equation}
\begin{aligned}
    &\E_{\mathcal{C}}\|X^{(Q)}_{t}-X^\ast\|^2=\left[\frac{t}{\ell}\left(\frac{13L\eta^2}{\ell\zeta}\left(1-\frac{\delta^2\sqrt{\omega}}{82}\right)^{Q}\|X_0-X^\ast\|^2+\|X_0\|^2\right)+\|X_0-X^\ast\|^2\right]\ell^t\\
    &\quad+\frac{20\eta^2}{1-\ell}\frac{\Delta^2}{1-\xi}\left(1-\frac{\delta^2\sqrt{\omega}}{82}\right)^{Q}.
    \end{aligned}
\end{equation}
Finally to obtain the stated result we make a few approximations. First, we assume $\hat{L}+\hat{\mu}\geq L+\mu$ and set $\eta = 1/\hat{L}$. This in turn means
\begin{equation}
    \ell = 1-\frac{\mu}{\hat{L}},\quad \xi:=23(1-\frac{\delta\gamma}{2})^{Q}.
\end{equation}
We further assume that the nodes are initialized at zero such that $\|X_0\|^2=0$. Furthermore, assuming
\begin{equation}
23(1-\frac{\delta\gamma}{2})^{Q}<\frac{\ell}{2},
\end{equation}
we can write
\begin{equation}
\zeta>\frac{\ell}{2}, \quad 1-\xi > 1-\ell,
\end{equation}
thereby establishing the stated bound.